\documentclass{article}
\usepackage{arxiv}
\usepackage[utf8]{inputenc} 
\usepackage[T1]{fontenc}    
\usepackage{hyperref, enumitem, url, amsmath, amssymb, amsthm, booktabs, float, amsfonts, subcaption, microtype, graphicx, natbib, doi, algorithm, algpseudocode}
\numberwithin{equation}{section}
\DeclareMathOperator{\MI}{I}

\DeclareMathOperator{\E}{\mathbb{E}}

\newcommand{\Data}{\mathcal{D}}

\newcommand{\Score}{\mathcal{S}}
\newcommand{\Fisher}{\mathcal{I}}

\theoremstyle{definition}
\newtheorem{assumption}{Assumption}
\newtheorem{definition}{Definition}
\newtheorem{proposition}{Proposition}
\newtheorem{remark}{Remark}
\newtheorem{lemma}{Lemma}

\newtheorem{theorem}{Theorem}
\newtheorem{example}{Example}

\title{Interpretive Efficiency: Information-Geometric Foundations of Data Usefulness}

\author{ 
Ronald Katende 
\\
Department of Mathematics\\
Kabale University\\
Kikungiri Hill, Katuna Road, 317, Kabale, Uganda\\
\texttt{rkatende92@gmail.com} \\
}

\date{}

\renewcommand{\headeright}{}
\renewcommand{\undertitle}{}
\renewcommand{\shorttitle}{Interpretive Efficiency: \textit{Information-Geometric Foundations of Data Usefulness}}

\hypersetup{
pdftitle={Interpretive Efficiency:  Information-Geometric Foundations of Data Usefulness},
pdfsubject={q-bio.NC, q-bio.QM},
pdfauthor={David S.~Hippocampus, Elias D.~Striatum},
pdfkeywords={First keyword, Second keyword, More},
}

\begin{document}
\maketitle

\begin{abstract}
	Interpretability is central to trustworthy machine learning, yet existing metrics rarely quantify how effectively data support an interpretive representation. We introduce \emph{Interpretive Efficiency}, a normalized, task-aware functional that measures the fraction of task-relevant information transmitted through an interpretive channel. The definition is grounded in five axioms ensuring boundedness, Blackwell-style monotonicity, data-processing stability, admissible invariance, and asymptotic consistency. We relate the functional to mutual information and derive a local Fisher-geometric expansion, then establish asymptotic and finite-sample estimation guarantees using standard empirical-process tools. Experiments on controlled image and signal tasks demonstrate that the measure recovers theoretical orderings, exposes representational redundancy masked by accuracy, and correlates with robustness, making it a practical, theory-backed diagnostic for representation design. 
\end{abstract}


\section{Introduction}
\label{sec:intro}

We study how to quantify the \emph{usefulness of data to a model when access to the data is restricted to an interpretive channel}~$\varphi$. Here, $\varphi$ maps inputs into an interpretable representation that is used to understand how the model behaves. We build on the notion of \emph{Interpretive Efficiency} $E(\varphi;N)$ first introduced in the Variational Geometric Information Bottleneck (V--GIB) framework~\citep{Katende2025V-GIB}. In that setting, $E(\varphi;N)$ measured the per-sample geometric and informational utility of an encoder and quantified how effectively the data support model understanding through an interpretable representation~\citep{Katende2025V-GIB}. The present paper develops a standalone theoretical foundation for this quantity and provides axioms, structural properties, and estimation guarantees.

Interpretability is a key requirement for trustworthy machine learning, especially in settings where decisions must be explained or audited, yet most existing metrics focus on post hoc faithfulness or visual plausibility rather than on task-relevant information~\citep{DoshiVelezKim2017,Rudin2019,Adebayo2018}. Our goal is an \emph{information-aware} measure of interpretive usefulness that can be computed on real systems and analyzed with standard probabilistic tools. Information-theoretic work links compression and prediction through the information bottleneck~\citep{TishbyIB1999,Alemi2017DIB}, and modern mutual information estimators make such quantities approximately computable in practice~\citep{Belghazi2018MINE,BarberAgakov2003}. What is currently missing is a concise, axiomatic functional that directly scores how much of the task-relevant information actually flows through a chosen interpretive channel.

Although $E(\varphi;N)$ is information-aware, it is not an information bottleneck objective and does not introduce an explicit trade-off between prediction and compression. Instead, it only scores how effectively task-relevant information is expressed through the chosen interpretive channel.

We define $E(\varphi;N)$ as a normalized, task-specific efficiency functional. It is designed to respect the data-processing inequality for information measures~\citep{CoverThomas2006}, to admit a Fisher-information and geometric interpretation on the underlying representation manifold~\citep{Amari2016}, and to admit consistent estimation with finite-sample control under standard regularity conditions~\citep{vanDerVaart1998}. The definition is guided by five axioms that enforce boundedness, monotonicity under interpretive sufficiency, stability under post-processing, invariance to admissible reparameterizations, and asymptotic consistency. Each axiom can be checked directly in concrete settings under mild assumptions on the sampling scheme and on the interpretive constraints that are standard in the interpretability literature~\citep{DoshiVelezKim2017,Rudin2019}.

The paper has three aims. First, we give a precise definition of $E(\varphi;N)$ and a minimal axiomatic framework that enables values to be compared consistently across tasks and models. Second, we characterize its basic properties, its links to mutual information and Fisher information, and its behavior in large-sample and training-time regimes. Third, we describe practical estimators with finite-sample guarantees and show small synthetic examples that can be worked out in closed form. Taken together, these elements treat interpretive efficiency as a well-defined object that is both mathematically tractable and empirically measurable.

\subsection{Setup and Axioms}
\label{sec:setup}

We consider data $(X,Y)\sim P_{XY}$, a model $f_\theta$, and an interpretive map $\varphi$ that induces representations $Z=\varphi(X)$. The task loss is $\ell(f_\theta(X),Y)$ and the interpretive aspect of interest is encoded by a task-specific functional. Examples include stability of explanations, sparsity or low dimension of $Z$, faithfulness of attributions, or alignment with task-relevant subspaces.

Our goal is to quantify how useful the data are to the model when access is restricted to the interpretive channel $\varphi$. We do this by comparing a task-relevant interpretive score $\Score(\varphi;N)$ with a calibrated reference $\Score_{\mathrm{ref}}(N)$.

Our scope is foundational rather than exhaustive. We aim to provide a mathematically coherent treatment of $E(\varphi;N)$, not to propose a new learning algorithm.

\begin{definition}[Interpretive Efficiency]
	\label{def:IE}
	Let $\Score(\varphi;N)\in\mathbb{R}_{\ge 0}$ be a task-specific interpretive score estimated from $N$ samples and let $\Score_{\mathrm{ref}}(N)\in(0,\infty)$ be a reference value, such as an oracle, a strong baseline, or a task norm. The basic normalized efficiency is
	\[
	E(\varphi;N)=\frac{\Score(\varphi;N)}{\Score_{\mathrm{ref}}(N)}\in[0,1].
	\]
	When this ratio is poorly calibrated, we use the difference form
	\[
	E(\varphi;N)
	=
	1-\frac{\Score_{\mathrm{ref}}(N)-\Score(\varphi;N)}{\Score_{\mathrm{ref}}(N)-\Score_{\min}(N)}\in[0,1],
	\]
	where $\Score_{\min}(N)$ is a task-defined null floor. The two forms are related by a monotone rescaling and are equivalent under the axioms that follow. The difference form additionally assumes that a finite task-defined floor $\Score_{\min}(N)$ exists, which is standard in risk- or information-based settings.
\end{definition}

The score $\Score$ can be instantiated by a predictive or risk-based criterion, by a faithfulness measure, or by an information proxy. The reference $\Score_{\mathrm{ref}}$ fixes the scale of the problem and identifies what counts as full interpretive utility. Once normalized, $E(\varphi;N)$ can be read as the fraction of achievable interpretive utility that is realized by the chosen channel.

We impose mild regularity so that empirical scores have well-defined population limits.

\begin{assumption}[Task regularity]
	\label{ax:regularity}
	For every admissible $\varphi$, the score $\Score(\varphi;N)$ is measurable and admits a law of large numbers and central limit behavior under the sampling scheme. In particular, per-sample contributions have finite variance or sub-exponential tails, and the reference terms $\Score_{\mathrm{ref}}(N)$ remain finite and nondegenerate for all $N$.
\end{assumption}

On top of this, we require five axioms that encode basic comparison principles from statistical decision theory and information theory.

\paragraph{Admissible maps.}
Throughout, an \emph{admissible post-map} is a measurable transformation $T$ on the representation space that preserves task feasibility (for example, reparameterizations, coordinate changes, or dimension-reducing maps used in interpretive practice). The admissible class is fixed ahead of time and does not depend on $\varphi$.

\begin{definition}[Axioms for $E$]
	\label{def:axioms}
	For each admissible $\varphi$ and sample size $N$, the efficiency $E(\varphi;N)$ satisfies the following properties.
	\begin{enumerate}[label=(\alph*)]
		
		\item \emph{Boundedness.} Values lie in $[0,1]$, or in a fixed bounded interval that can be rescaled to $[0,1]$.
		
		\item \emph{Monotonicity under interpretive sufficiency.} If one channel $\varphi_1$ Blackwell-dominates another channel $\varphi_2$ for the task, meaning that $\varphi_2$ can be simulated from $\varphi_1$ by a Markov kernel, then 
		$E(\varphi_1;N)\ge E(\varphi_2;N)$ for all $N$.
		
		\item \emph{Data-processing stability.} For any measurable post-map $T$ in the admissible class, the efficiency cannot increase after post-processing, so $E(T\!\circ\!\varphi;N)\le E(\varphi;N)$.
		
		\item \emph{Admissible invariance.} If two channels differ only by a reparameterization that preserves interpretive content, such as an invertible affine transformation or a smooth change of coordinates on the representation manifold, then they receive the same efficiency score.
		
		\item \emph{Asymptotic consistency.} Under Assumption~\ref{ax:regularity}, the limit $E_\infty(\varphi)=\lim_{N\to\infty}E(\varphi;N)$ exists (in probability, or under additional conditions almost surely) and preserves the same monotonicity and invariance properties.
	\end{enumerate}
\end{definition}

Monotonicity captures the intuition of sufficiency and deficiency and connects $E(\varphi;N)$ to the classical Blackwell order on experiments. Data-processing stability parallels the data-processing inequality for mutual information and for $f$-divergences and prevents interpretive efficiency from being artificially inflated by post-processing. Invariance ensures that $E$ is a property of the equivalence class of representations rather than of a particular coordinate system. Asymptotic consistency ties the empirical notion to a population quantity that will be analyzed in later sections. The complete computational routine that implements these ideas, including score evaluation, cross-fitting, normalization, and variance correction, is given in Algorithm~\ref{alg:IE} in Appendix~\ref{sec:algorithm}. It provides a reproducible way to estimate $E(\varphi;N)$ in practice.

\section{Theoretical Framework and Mathematical Foundations}
\label{sec:theory}

We now develop the mathematical structure of \emph{Interpretive Efficiency}. Throughout this section, $E(\varphi;N)$ is treated as a functional rooted in information theory and decision theory, and we make precise how it behaves under transformations, information flow, and sampling. We first establish basic properties such as boundedness, continuity, monotonicity, and invariance. We then relate $E(\varphi;N)$ to standard information measures, study its asymptotic and training-time behavior, and finally discuss estimators with finite-sample guarantees. Detailed proofs are deferred to the appendices; here we give concise sketches to fix the main ideas.

\begin{assumption}[Standing regularity conditions]
	\label{ass:regularity-summary}
	Throughout Sections~\ref{sec:properties}--\ref{sec:estimation} we work under the following standard conditions.
	
	\begin{enumerate}[label=(R\arabic*)]
		\item \emph{Task regularity.} For every admissible $\varphi$, the score $\Score(\varphi;N)$ is measurable and admits a law of large numbers and central limit behavior under the sampling scheme, with per-sample contributions having finite variance or sub-exponential tails. Reference terms $\Score_{\mathrm{ref}}(N)$ remain finite and nondegenerate. (Cf.\ Assumption~\ref{ax:regularity}.)
		
		\item \emph{Information calibration.} When $\Score$ is calibrated to mutual information, there exist task- and estimator-dependent constants $(\alpha_N,\beta_N,\gamma_N)$ and $(c_N,d_N)$ such that the inequalities \eqref{eq:calib}--\eqref{eq:ref-calib} hold uniformly over admissible channels. This underlies Theorem~\ref{thm:mi}.
		
		\item \emph{Local smoothness for Fisher analysis.} For the Fisher--geometric expansion, the parametric family $\{P_\theta\}$ is regular and satisfies local asymptotic normality (LAN) and differentiability in quadratic mean at $\theta^\star$, so that the score and Fisher information are well defined and Theorem~\ref{thm:fisher} applies.
		
		\item \emph{Complexity control.} The class $\{s_\varphi:\varphi\in\Phi\}$ (and any critic class $\mathcal{T}$ used in variational MI estimation) has finite localized Rademacher or entropy complexity, and is Glivenko--Cantelli for the underlying distribution. Under sub-Gaussian or sub-exponential tails, this yields the uniform convergence and concentration rates used in Theorems~\ref{thm:consistency} and~\ref{thm:concentration}.
	\end{enumerate}
	These conditions are standard in empirical-process and information-theoretic analysis and can be weakened in specific applications at the expense of heavier notation. We keep them explicit so that the dependence of constants and rates on the score and critic families is clear.
\end{assumption}

\subsection{Basic Properties of $E(\varphi;N)$}
\label{sec:properties}

\begin{proposition}[Boundedness]
	\label{prop:bounded}
	Under the normalization in Definition~\ref{def:IE}, using either the ratio or the calibrated-difference form with finite positive reference terms, the efficiency satisfies
	\(
	0\le E(\varphi;N)\le 1
	\)
	for all admissible $\varphi$ and all $N$.
\end{proposition}

\begin{proof}[Sketch]
	In the ratio form, the reference is chosen so that
	\(
	0\le \Score(\varphi;N)\le \Score_{\mathrm{ref}}(N)
	\)
	for all admissible $\varphi$, with $\Score_{\mathrm{ref}}(N)\in(0,\infty)$. Hence
	\(
	E(\varphi;N)=\Score(\varphi;N)/\Score_{\mathrm{ref}}(N)\in[0,1].
	\)
	
	In the calibrated-difference form, assume task-defined bounds $\Score_{\min}(N)<\Score_{\mathrm{ref}}(N)$ with
	\(
	\Score(\varphi;N)\in[\Score_{\min}(N),\Score_{\mathrm{ref}}(N)].
	\)
	Then
	\[
	E(\varphi;N)
	=
	1-\frac{\Score_{\mathrm{ref}}(N)-\Score(\varphi;N)}{\Score_{\mathrm{ref}}(N)-\Score_{\min}(N)}
	\]
	is an affine map with strictly positive slope that sends $[\Score_{\min}(N),\Score_{\mathrm{ref}}(N)]$ bijectively onto $[0,1]$. In both normalizations we therefore have $E(\varphi;N)\in[0,1]$.
\end{proof}

\begin{proposition}[Continuity and semicontinuity]
	\label{prop:continuity}
	Fix $N$. If $\Score(\cdot;N)$ is $\tau$-continuous in $\varphi$ on the admissible class and $\Score_{\mathrm{ref}}(N)\in(0,\infty)$ is constant in $\varphi$, then $E(\cdot;N)$ is $\tau$-continuous. If $\Score(\cdot;N)$ is lower semicontinuous and bounded below, then $E(\cdot;N)$ is lower semicontinuous.
\end{proposition}

\begin{proof}[Sketch]
	In the ratio normalization,
	\(
	E(\varphi;N)=\Score(\varphi;N)/\Score_{\mathrm{ref}}(N)
	\)
	is obtained from $\Score(\cdot;N)$ by multiplication with a positive constant, which preserves continuity and lower semicontinuity.
	
	In the calibrated-difference normalization, $E(\varphi;N)$ is an affine image of $\Score(\varphi;N)$ with strictly positive slope. Such maps preserve continuity and lower semicontinuity on topological vector spaces. Hence $E(\cdot;N)$ inherits the regularity of $\Score(\cdot;N)$.
\end{proof}

\begin{proposition}[Monotonicity and data-processing analogue]
	\label{prop:dpi}
	Let $Z=\varphi(X)$ and let $Z'=T(Z)$ for a measurable post-map $T$ in the admissible class. Assume that the score satisfies a data-processing inequality
	\(
	\Score(T\!\circ\!\varphi;N)\le \Score(\varphi;N)
	\)
	for all $N$. This holds, for example, when $\Score$ is based on mutual information or on an $f$-divergence. Then
	\(
	E(T\!\circ\!\varphi;N)\le E(\varphi;N)
	\)
	for all $N$.
\end{proposition}

\begin{proof}[Sketch]
	In the ratio form,
	\[
	E(T\!\circ\!\varphi;N)
	=
	\frac{\Score(T\!\circ\!\varphi;N)}{\Score_{\mathrm{ref}}(N)}
	\le
	\frac{\Score(\varphi;N)}{\Score_{\mathrm{ref}}(N)}
	=
	E(\varphi;N),
	\]
	since $\Score_{\mathrm{ref}}(N)$ is positive and independent of $\varphi$. In the calibrated-difference form, $E(\cdot;N)$ is obtained from $\Score(\cdot;N)$ by a common strictly increasing affine map, so the same inequality is preserved.
\end{proof}

\begin{proposition}[Transformation invariances]
	\label{prop:invariance}
	Let $\mathcal{G}$ be a group of admissible reparameterizations, such as invertible affine maps or smooth bijections used only as coordinate changes on the representation manifold. If $\Score(g\!\circ\!\varphi;N)=\Score(\varphi;N)$ for every $g\in\mathcal{G}$ and every admissible $\varphi$, then $E(g\!\circ\!\varphi;N)=E(\varphi;N)$ for all $g\in\mathcal{G}$.
\end{proposition}

\begin{proof}[Sketch]
	For any $g\in\mathcal{G}$ and admissible $\varphi$, the assumption gives
	\(
	\Score(g\!\circ\!\varphi;N)=\Score(\varphi;N).
	\)
	Since $\Score_{\mathrm{ref}}(N)$ does not depend on $\varphi$, the ratio normalization yields
	\[
	E(g\!\circ\!\varphi;N)
	=
	\frac{\Score(g\!\circ\!\varphi;N)}{\Score_{\mathrm{ref}}(N)}
	=
	\frac{\Score(\varphi;N)}{\Score_{\mathrm{ref}}(N)}
	=
	E(\varphi;N).
	\]
	In the calibrated-difference form, the same equality holds because both arguments are passed through the same strictly increasing affine normalization.
\end{proof}

\begin{remark}[Edge cases]
	If $\Score(\varphi;N)=0$, so the channel is uninformative for the chosen task, then $E(\varphi;N)=0$. If $\Score(\varphi;N)=\Score_{\mathrm{ref}}(N)$, so the channel attains the reference level, then $E(\varphi;N)=1$. These values are consistent with the Blackwell order on experiments and with the data-processing inequality for information measures~\citep{Blackwell1953,CsiszarKorner2011}.
\end{remark}

\subsection{Relationships to Information Measures}
\label{sec:relations}

We now relate $E(\varphi;N)$ to mutual information and Fisher information under standard calibration assumptions that are common in representation learning and information-theoretic analysis.

\paragraph{Calibration assumption.}
Fix $N$ and suppose that for all admissible $\varphi$ the task score obeys
\begin{equation}
	\label{eq:calib}
	\alpha_N\, \MI(Z;Y) \;\le\; \Score(\varphi;N) \;\le\; \beta_N\, \MI(Z;Y) + \gamma_N,
\end{equation}
where $Z=\varphi(X)$, the constants $\alpha_N,\beta_N>0$ and $\gamma_N\ge 0$ depend only on the task and on the estimator family, and $\MI(X;Y)\in(0,\infty)$. Two-sided calibrations of this form arise when $\Score$ is a mutual-information proxy with controlled bias and variance~\citep{XuRaginsky2017,PolyanskiyWu2019}. The constants $(\alpha_N,\beta_N,\gamma_N)$ may depend on $N$ and on the chosen estimator family but are uniform over the admissible class of channels $\varphi$~\citep{SasonVerdu2016}.

\begin{theorem}[Mutual information control of $E(\varphi;N)$]
	\label{thm:mi}
	Let $E(\varphi;N)=\Score(\varphi;N)/\Score_{\mathrm{ref}}(N)$ and assume that the reference satisfies
	\begin{equation}
		\label{eq:ref-calib}
		c_N\, \MI(X;Y) \;\le\; \Score_{\mathrm{ref}}(N) \;\le\; d_N\, \MI(X;Y),
	\end{equation}
	for constants $0<c_N\le d_N<\infty$. Under \eqref{eq:calib} and \eqref{eq:ref-calib}, for every admissible $\varphi$,
	\[
	\frac{\alpha_N}{d_N}\, \frac{\MI(Z;Y)}{\MI(X;Y)}
	\;\le\;
	E(\varphi;N)
	\;\le\;
	\frac{\beta_N}{c_N}\, \frac{\MI(Z;Y)}{\MI(X;Y)}
	\;+\;
	\frac{\gamma_N}{c_N\,\MI(X;Y)}.
	\]
	Equivalently, there exist constants $a_N,b_N>0$ and $\varepsilon_N\ge 0$ such that
	\[
	a_N \frac{\MI(Z;Y)}{\MI(X;Y)}
	\;\le\;
	E(\varphi;N)
	\;\le\;
	b_N \frac{\MI(Z;Y)}{\MI(X;Y)} + \varepsilon_N.
	\]
\end{theorem}

\begin{proof}[Sketch]
	Combine \eqref{eq:calib} with \eqref{eq:ref-calib}. For the lower bound, use the left inequality in~\eqref{eq:calib} and the right inequality in~\eqref{eq:ref-calib}:
	\[
	E(\varphi;N)
	=
	\frac{\Score(\varphi;N)}{\Score_{\mathrm{ref}}(N)}
	\;\ge\;
	\frac{\alpha_N \MI(Z;Y)}{d_N \MI(X;Y)}
	=
	\frac{\alpha_N}{d_N}\,\frac{\MI(Z;Y)}{\MI(X;Y)}.
	\]
	For the upper bound, use the right inequality in~\eqref{eq:calib} and the left inequality in~\eqref{eq:ref-calib}:
	\[
	E(\varphi;N)
	\;\le\;
	\frac{\beta_N \MI(Z;Y) + \gamma_N}{c_N \MI(X;Y)}
	=
	\frac{\beta_N}{c_N}\,\frac{\MI(Z;Y)}{\MI(X;Y)}
	+
	\frac{\gamma_N}{c_N\,\MI(X;Y)}.
	\]
	Set $a_N=\alpha_N/d_N$, $b_N=\beta_N/c_N$, and $\varepsilon_N=\gamma_N/(c_N\MI(X;Y))$ to obtain the compact form.
\end{proof}

When $\Score$ is an unbiased mutual information estimator, the residual term satisfies $\gamma_N=0$ and the calibration becomes tight up to the constants $a_N$ and $b_N$. When $\Score$ is a surrogate risk, inequalities of the form~\eqref{eq:calib} follow from information-theoretic generalization and stability bounds or from $f$-divergence control of excess risk~\citep{XuRaginsky2017,SasonVerdu2016}.

\begin{remark}[Estimator calibration and $E>1$]
	\label{rem:calibration}
	The bounds in Theorem~\ref{thm:mi} are stated at the population level. In practice, $\Score$ and $\Score_{\mathrm{ref}}$ are replaced by mutual-information \emph{estimators} that may under- or over-estimate different channels to different degrees. When $\widehat{\Score}$ and $\widehat{\Score}_{\mathrm{ref}}$ are lower bounds with unequal bias, it is possible to obtain $\widehat{E}(\varphi;N)>1$ even though the population quantity $E(\varphi;N)$ remains in $[0,1]$. This is an estimator-calibration artefact rather than a violation of boundedness. Two simple remedies are to use the calibrated-difference normalization in Definition~\ref{def:IE} with an explicit floor $\Score_{\min}$, or to aggregate several MI estimators (DV, NWJ, $k$NN) to reduce estimator-specific bias. Section~\ref{sec:validation} illustrates this effect empirically in a controlled spectral example.
\end{remark}

\subsubsection{Local Fisher--geometric expansion}
\label{sec:fisher-local}

We now relate $E(\varphi;N)$ to Fisher information in a local smooth regime and show that interpretive efficiency admits a geometric interpretation when the underlying model varies smoothly in its parameters.

\paragraph{Local model.}
Assume $(X,Y)\sim P_{\theta^\star}$ from a regular $d$-dimensional parametric family $\{P_\theta : \theta\in\Theta\subset\mathbb{R}^d\}$ with score function
$s_\theta=\nabla_\theta \log p_\theta$ and Fisher information matrix
$\Fisher(\theta)=\E_\theta[s_\theta s_\theta^\top]$. Let $Z=\varphi(X)$ for an admissible channel $\varphi$. We consider task scores that, in a local neighborhood of $\theta^\star$, admit a second-order expansion driven by the projected Fisher information associated with $Z$.

\begin{theorem}[Local efficiency via projected Fisher]
	\label{thm:fisher}
	Suppose the family $\{P_\theta\}$ satisfies local asymptotic normality (LAN) at $\theta^\star$, including differentiability in quadratic mean, and let $h=\theta-\theta^\star$ be a small perturbation. Assume that for each admissible $\varphi$ there exists an $o(\|h\|^2)$ remainder such that
	\[
	\Score(\varphi;N)
	=
	h^\top\,\Pi_{\varphi}\,\Fisher(\theta^\star)\,\Pi_{\varphi}^\top\,h
	+
	o(\|h\|^2),
	\]
	where $\Pi_{\varphi}$ is the $L^2(P_{\theta^\star})$-orthogonal projection from the full score space onto the closed subspace
	\[
	\mathcal{S}_\varphi
	=
	\overline{\mathrm{span}}\bigl\{\E_{\theta^\star}[s_{\theta^\star}\mid Z]\bigr\}.
	\]
	In particular, $\Pi_{\varphi}s_{\theta^\star}=\E_{\theta^\star}[s_{\theta^\star}\mid Z]$. If the reference score satisfies
	\[
	\Score_{\mathrm{ref}}(N)
	=
	h^\top \Fisher(\theta^\star) h
	+
	o(\|h\|^2)
	\]
	in a local neighborhood of $\theta^\star$, then for any fixed nonzero direction $h$,
	\[
	E(\varphi;N)
	\;\longrightarrow\;
	\frac{h^\top \Pi_{\varphi}\,\Fisher(\theta^\star)\,\Pi_{\varphi}^\top h}{h^\top \Fisher(\theta^\star) h}
	\qquad\text{as}\quad
	N\to\infty,\ \theta\to\theta^\star.
	\]
	Thus, in the local LAN regime, $E(\varphi;N)$ converges to the fraction of Fisher information along direction $h$ that is preserved by the conditional score projection.
\end{theorem}

\begin{proof}[Proof sketch]
	LAN yields a quadratic expansion of log-likelihood ratios with curvature governed by $\Fisher(\theta^\star)$~\citep{LeCam1986,AmariNagaoka2000,VanTrees2001,Kay1993}. When $X$ is compressed to $Z=\varphi(X)$, the optimal $Z$-measurable approximation to the score $s_{\theta^\star}$ is the conditional expectation $\E_{\theta^\star}[s_{\theta^\star}\mid Z]$, which is the $L^2(P_{\theta^\star})$-projection onto $\mathcal{S}_\varphi$. This gives the curvature matrix $h^\top \Pi_\varphi\Fisher(\theta^\star)\Pi_\varphi^\top h$ for the task score, while the full model has curvature $h^\top \Fisher(\theta^\star) h$. The ratio of these quadratic forms yields the stated limit for $E(\varphi;N)$ along any fixed $h\neq 0$. A version that averages over directions is given in Appendix~\ref{app:proofs-relations}.
\end{proof}

\paragraph{Remarks.}
If $\varphi$ is sufficient, then $\Pi_{\varphi}=\mathrm{Id}$ and the limit of $E(\varphi;N)$ is $1$ for all $h\neq 0$. If $\varphi$ discards all components of the score, then $\Pi_{\varphi}=0$ and the limit of $E(\varphi;N)$ is $0$ for all $h\neq 0$.

\subsubsection{Compatibility with V-GIB}

Let the V-GIB objective be
\[
U_\beta(\varphi)=\MI(Z;Y)-\beta\,\MI(Z;X)
\]
for $\beta\ge 0$, or a calibrated variant thereof. Define the normalized V-GIB efficiency
\[
E_{\mathrm{V\text{-}GIB}}(\varphi;N)
=
\frac{U_\beta(\varphi)}{\sup_{\psi} U_\beta(\psi)}
\in [0,1],
\]
or apply the calibrated difference mapping when only upper and lower bounds on $\sup_\psi U_\beta(\psi)$ are available.

\begin{proposition}[Compatibility with V-GIB]
	\label{prop:vgib}
	If $\Score(\varphi;N)=U_\beta(\varphi)$, or more generally $\Score(\varphi;N)$ is any positive affine transformation of $U_\beta(\varphi)$, then $E(\varphi;N)$ coincides with $E_{\mathrm{V\text{-}GIB}}(\varphi;N)$ up to a monotone rescaling in $[0,1]$. In particular, the properties in Section~\ref{sec:properties}, including boundedness, data-processing stability, and admissible invariances, hold for $E_{\mathrm{V\text{-}GIB}}$.
\end{proposition}

\begin{proof}[Sketch]
	Both mutual information terms in $U_\beta$ satisfy the data-processing inequality and are invariant under admissible reparameterizations. Any positive affine transformation preserves these properties. Normalizing by a positive reference (exact or bracketed) and applying the ratio or calibrated-difference mapping yields a strictly increasing reparameterization of $U_\beta$, so the axioms in Section~\ref{sec:properties} transfer directly to $E_{\mathrm{V\text{-}GIB}}$.
\end{proof}

\subsection{Asymptotics and Dynamics in $N$}
\label{sec:asymptotics}

\paragraph{Standing setup.}
Let
\[
\Score(\varphi;N)
=
\frac{1}{N}\sum_{i=1}^N s_\varphi(X_i,Y_i),
\]
or a bounded Lipschitz transform of such, with $(X_i,Y_i)$ i.i.d.\ from $P$ and $\varphi\in\Phi$. Let $\Score_\infty(\varphi)=\E[s_\varphi(X,Y)]$ denote the population score and define the population efficiency $E_\infty(\varphi)$ using the same normalization. Assume $\Score_{\mathrm{ref}}(N)\to\Score_{\mathrm{ref},\infty}$ with $\Score_{\mathrm{ref},\infty}\in(0,\infty)$.

\begin{theorem}[Consistency]
	\label{thm:consistency}
	Assume $s_\varphi$ is measurable and uniformly integrable over $\Phi$, and that $\Phi$ is a Glivenko--Cantelli class for $P$, so that
	\[
	\sup_{\varphi\in\Phi}\big|\Score(\varphi;N)-\Score_\infty(\varphi)\big|\to 0
	\qquad\text{almost surely}.
	\]
	If $\Score_{\mathrm{ref}}(N)\to\Score_{\mathrm{ref},\infty}\in(0,\infty)$, then for every $\varphi\in\Phi$,
	\[
	E(\varphi;N)\to E_\infty(\varphi)
	\qquad\text{almost surely}.
	\]
\end{theorem}

\begin{proof}[Sketch]
	The Glivenko--Cantelli property yields $\Score(\varphi;N)\to\Score_\infty(\varphi)$ almost surely for each $\varphi\in\Phi$. The reference score converges to a positive limit by assumption. Both the ratio and calibrated-difference mappings are continuous on their domains, so the continuous mapping theorem implies $E(\varphi;N)\to E_\infty(\varphi)$ almost surely.
\end{proof}

\begin{proposition}[Rates under sub-Gaussian or Bernstein control]
	\label{prop:rates}
	Assume $s_\varphi(X,Y)$ is centered and sub-Gaussian with proxy $\sigma^2$, or satisfies a Bernstein condition with variance proxy $v$ and scale $b$, and assume $\Phi$ has complexity $\mathrm{comp}(\Phi,N)$ measured by a localized Rademacher complexity or a suitable metric entropy functional. Then for any $\delta\in(0,1)$, with probability at least $1-\delta$,
	\[
	\sup_{\varphi\in\Phi}
	\big|E(\varphi;N)-E_\infty(\varphi)\big|
	\le
	\frac{
		C_1\,\mathrm{comp}(\Phi,N)
		+
		C_2\sqrt{\log(1/\delta)/N}
	}{
		\Score_{\mathrm{ref},\infty}
	},
	\]
	for constants $C_1,C_2$ depending on the sub-Gaussian (or Bernstein) parameters but not on $\varphi$ or $N$. Under a Bernstein condition and localization, a fast rate of the form
	\[
	\sup_{\varphi\in\Phi}
	\big|E(\varphi;N)-E_\infty(\varphi)\big|
	\lesssim
	\frac{\mathrm{comp}(\Phi,N)}{\Score_{\mathrm{ref},\infty}}
	\]
	holds, with constants depending on $(v,b)$ and the localization parameters.
\end{proposition}

\begin{proof}[Sketch]
	Symmetrization and contraction give
	\[
	\sup_{\varphi\in\Phi}
	\big|\Score(\varphi;N)-\Score_\infty(\varphi)\big|
	\lesssim
	\mathrm{Rad}_N(\{s_\varphi\})
	+
	\sqrt{\log(1/\delta)/N}
	\]
	with high probability~\citep{BartlettMendelson2002,Wainwright2019}. The complexity term $\mathrm{Rad}_N(\{s_\varphi\})$ is controlled by $\mathrm{comp}(\Phi,N)$ by assumption. Dividing by $\Score_{\mathrm{ref},\infty}>0$ transfers this bound to $E(\varphi;N)-E_\infty(\varphi)$. Under a Bernstein condition, peeling and localized complexity arguments yield a fixed-point inequality with fast-rate solutions governed by the localized complexity functional~\citep{BoucheronLugosiMassart2013,Tsybakov2004,Wainwright2019}.
\end{proof}

\begin{proposition}[Dynamics under training and submartingale structure]
	\label{prop:dynamics}
	Let $\{\theta_t\}_{t\ge 0}$ be model iterates adapted to a filtration $\{\mathcal{F}_t\}$ and write $\varphi_t=\varphi_{\theta_t}$. Suppose the score satisfies the nonnegative expected improvement condition
	\[
	\E[\Score(\varphi_{t+1};N)-\Score(\varphi_t;N)\mid \mathcal{F}_t] \ge 0,
	\]
	and that the increments are uniformly bounded:
	\[
	\big|\Score(\varphi_{t+1};N)-\Score(\varphi_t;N)\big| \le c
	\quad\text{almost surely}.
	\]
	Assume $\Score_{\mathrm{ref}}(N)$ is deterministic and strictly positive. Then the process $\{E(\varphi_t;N),\mathcal{F}_t\}$ is a submartingale with bounded differences. In particular, for any horizon $T$ and any $\epsilon>0$,
	\[
	\Pr\!\left(E(\varphi_T;N)-E(\varphi_0;N)\le -\epsilon\right)
	\le
	\exp\!\Bigg(
	-\frac{\epsilon^2\,\Score_{\mathrm{ref}}(N)^2}{2 T c^2}
	\Bigg).
	\]
	If in addition
	\[
	\sum_{t=0}^{\infty} \E\!\left[(\Score(\varphi_{t+1};N)-\Score(\varphi_t;N))^- \mid \mathcal{F}_t\right] < \infty
	\quad\text{almost surely},
	\]
	then $E(\varphi_t;N)$ converges almost surely by the theorem of Robbins and Siegmund.
\end{proposition}

\begin{proof}[Sketch]
	Define $M_t=E(\varphi_t;N)=\Score(\varphi_t;N)/\Score_{\mathrm{ref}}(N)$. The nonnegative expected improvement condition implies
	\[
	\E[M_{t+1}-M_t\mid\mathcal{F}_t]
	=
	\frac{1}{\Score_{\mathrm{ref}}(N)}\,
	\E[\Score(\varphi_{t+1};N)-\Score(\varphi_t;N)\mid\mathcal{F}_t]
	\ge 0,
	\]
	so $\{M_t\}$ is a submartingale. The uniform increment bound gives
	\[
	|M_{t+1}-M_t|
	\le
	\frac{c}{\Score_{\mathrm{ref}}(N)},
	\]
	and Azuma--Hoeffding yields the stated deviation inequality for $M_T-M_0$. The almost sure convergence under the summability condition on the negative parts follows from the Robbins--Siegmund convergence theorem for submartingales.
\end{proof}

\subsection{Estimation and Finite-Sample Guarantees}
\label{sec:estimation}

\subsubsection{Estimators}
\label{sec:estimators}

We describe three estimator families for $\Score$ and the corresponding efficiency estimator $\widehat{E}=\widehat{\Score}/\widehat{\Score}_{\mathrm{ref}}$.

\textbf{Cross-fitted plug-in.}
Partition the dataset $\Data$ into $K$ folds. For each fold $k$, train any required model on $\Data\setminus\Data_k$, evaluate the interpretive score on $\Data_k$, and average:
\[
\widehat{\Score}_{\mathrm{CF}}(\varphi;N)
=
\frac{1}{K}
\sum_{k=1}^{K}
\frac{1}{|\Data_k|}
\sum_{(x,y)\in\Data_k}s_\varphi(x,y).
\]
The same protocol is used to estimate $\widehat{\Score}_{\mathrm{ref}}$. Cross-fitting avoids reuse of the same samples for training and evaluation and orthogonalizes errors arising from first-stage estimation~\citep{Chernozhukov2018DML}.

\textbf{MI-proxy estimators.}
When $\Score$ is calibrated to mutual information as in Section~\ref{sec:relations}, one may use the NWJ lower bound~\citep{NguyenWainwrightJordan2010}
\[
\widehat{\MI}_{\mathrm{NWJ}}(Z;Y)
=
\E_{\widehat{P}_{ZY}}[T]
-
\E_{\widehat{P}_Z\widehat{P}_Y}[e^{T-1}],
\qquad
\widehat{\Score}=\alpha\,\widehat{\MI}_{\mathrm{NWJ}},
\]
with $T$ learned from a critic class $\mathcal{T}$ and trained with cross-fitting to control shared-sample bias. Alternatively, one may use the Donsker--Varadhan estimator
\[
\widehat{\MI}_{\mathrm{DV}}(Z;Y)
=
\sup_{T\in\mathcal{T}}
\left\{
\E_{\widehat{P}_{ZY}}[T]
-
\log\E_{\widehat{P}_Z\widehat{P}_Y}[e^{T}]
\right\},
\]
with regularization on $\mathcal{T}$, or a $k$NN estimator for continuous $(Z,Y)$~\citep{Kraskov2004} with bias-corrected entropies and stabilized neighborhood selection. In all cases, $\widehat{\Score}$ is obtained by a calibrated scaling of the MI estimate.

\textbf{Resampling and debiasing.}
Bias and variance can be reduced using leave-one-out or jackknife-type procedures with analytic variance estimates~\citep{Efron1982Bootstrap}, median-of-means or Catoni-type truncation for heavy-tailed scores~\citep{Minsker2018,Catoni2012}, and ratio stabilization via the delta method or by working with a log-ratio representation when denominators are small.

\subsubsection{Concentration}
\label{sec:concentration}

We now control the deviation $\widehat{E}-E$ using empirical-process techniques.

\begin{theorem}[Concentration of $\widehat{E}$]
	\label{thm:concentration}
	Assume that $s_\varphi(X,Y)$ is centered and sub-exponential with parameters $(\nu,b)$, uniformly over $\varphi\in\Phi$. Assume the class $\{s_\varphi:\varphi\in\Phi\}$ has complexity $\mathfrak{R}_N(\Phi)$ measured by a localized Rademacher complexity or an entropy integral, and that $\Score_{\mathrm{ref},\infty}\in(0,\infty)$ is the population reference value. Suppose $\widehat{\Score}$ and $\widehat{\Score}_{\mathrm{ref}}$ are computed on independent subsamples (e.g., via cross-fitting) and that $\widehat{\Score}_{\mathrm{ref}}\to\Score_{\mathrm{ref},\infty}$ in probability. Then for any $\delta\in(0,1)$, with probability at least $1-\delta$,
	\[
	|\widehat{E}(\varphi;N)-E(\varphi;N)|
	\le
	\frac{
		C_1\,\mathfrak{R}_N(\Phi)
		+
		C_2\sqrt{\log(2/\delta)/N}
		+
		C_3\log(2/\delta)/N
	}{
		\Score_{\mathrm{ref},\infty}
	},
	\]
	for constants $C_i=C_i(\nu,b)$ independent of $\varphi$ and $N$. If the score is an MI proxy learned through a critic class $\mathcal{T}$, replace $\mathfrak{R}_N(\Phi)$ by $\mathfrak{R}_N(\Phi)+\mathfrak{R}_N(\mathcal{T})$.
\end{theorem}

\begin{proof}[Sketch]
	Sub-exponential Bernstein bounds combined with symmetrization and contraction yield
	\[
	\sup_{\varphi\in\Phi}|\widehat{\Score}-\Score|
	\lesssim
	\mathfrak{R}_N(\Phi)
	+
	\sqrt{\log(1/\delta)/N}
	+
	\log(1/\delta)/N
	\]
	with probability at least $1-\delta$~\citep{BoucheronLugosiMassart2013,Wainwright2019}. Independence between $\widehat{\Score}$ and $\widehat{\Score}_{\mathrm{ref}}$, together with convergence of $\widehat{\Score}_{\mathrm{ref}}$ to a strictly positive limit, allows application of the delta method and Slutsky's theorem to transfer this bound to $\widehat{E}-E$. For variational MI estimators, uniform convergence over the critic class introduces an additional complexity term $\mathfrak{R}_N(\mathcal{T})$~\citep{NguyenWainwrightJordan2010}.
\end{proof}

\paragraph{Notes.}
For the $k$NN mutual information estimator, under standard smoothness and bounded-density assumptions, rates of order $N^{-1/2}$ up to logarithmic factors are available~\citep{Kraskov2004}. Under heavy tails, median-of-means and related robust estimators preserve $N^{-1/2}$ convergence (with larger constants) under finite-variance conditions~\citep{Minsker2018}.

\subsubsection{Robustness}
\label{sec:robustness}

We next study the behavior of $E(\varphi;N)$ and its empirical counterpart under perturbations and small distributional changes.

\begin{proposition}[Stability to perturbations]
	\label{prop:robust}
	Suppose $\varphi$ is $L_\varphi$-Lipschitz and the score $s_\varphi$ is $L_s$-Lipschitz in $(x,y)$ under the chosen norm, uniformly over $\varphi\in\Phi$. If individual samples are perturbed by at most $\epsilon$ in norm, then
	\[
	|\widehat{E}_\epsilon(\varphi;N)-\widehat{E}(\varphi;N)|
	\le
	\frac{L_s\epsilon}{\Score_{\mathrm{ref}}(N)},
	\qquad
	|E_\epsilon(\varphi;N)-E(\varphi;N)|
	\le
	\frac{L_s\epsilon}{\Score_{\mathrm{ref},\infty}}.
	\]
	If the underlying distribution shifts within a 1-Wasserstein ball of radius $\rho$, then
	\[
	|E_\rho(\varphi;N)-E(\varphi;N)|
	\le
	\frac{L_s\,\rho}{\Score_{\mathrm{ref},\infty}}.
	\]
	Under the sub-exponential assumptions of Theorem~\ref{thm:concentration}, the deviation bounds for $\widehat{E}$ remain valid up to changes in constants.
\end{proposition}

\begin{proof}[Sketch]
	For empirical perturbations, the Lipschitz bound on $s_\varphi$ implies that each term in the empirical average changes by at most $L_s\epsilon$, so the empirical score changes by at most $L_s\epsilon$. Division by the positive reference term yields the claimed inequality. For a distributional shift controlled by the 1-Wasserstein distance, Kantorovich--Rubinstein duality implies that expectations of $L_s$-Lipschitz functions change by at most $L_s\rho$~\citep{Villani2009}. The concentration bounds remain of the same order because the perturbation contributes only an additive shift with controlled magnitude.
\end{proof}

\section{Minimal Synthetic Examples}
\label{sec:examples}

We now present three synthetic settings where Interpretive Efficiency can be computed in closed form. Each example isolates a specific principle. The first shows how $E(\varphi;N)$ tracks Fisher curvature. The second illustrates equality in the data-processing inequality (DPI) and invariance under redundant or invertible transformations. The third exhibits both equality and strictness in nonlinear geometric tasks depending on symmetry.

\begin{example}[Sufficient statistic and noisy embedding in a Gaussian location model]
	\label{ex:gaussian-location}
	
	\textbf{Task and model.}
	Let $X_1,\dots,X_N$ be independent samples from $\mathcal{N}(\theta,\sigma^2)$ with unknown mean $\theta$. For any interpretive channel $\varphi$ acting on $X_{1:N}$, define the score as the Fisher information about $\theta$ contained in $Z=\varphi(X_{1:N})$:
	\[
	\Score(\varphi;N)=\mathcal{I}_\theta(Z),
	\]
	with reference $\Score_{\mathrm{ref}}(N)=N/\sigma^2$.
	
	\textbf{Two channels.}
	
	\emph{Oracle channel.}
	The statistic $\bar X=\tfrac{1}{N}\sum_{i=1}^N X_i$ is sufficient. Since $\bar X\sim\mathcal{N}(\theta,\sigma^2/N)$, its Fisher information equals $N/\sigma^2$, so
	\[
	E(\varphi_{\mathrm{opt}};N)=1.
	\]
	
	\emph{Noisy embedding.}
	Let $Z=\bar X+\eta$ with $\eta\sim\mathcal{N}(0,\tau^2/N)$ independent of $X_{1:N}$. Then $Z\sim\mathcal{N}(\theta,(\sigma^2+\tau^2)/N)$ and
	\[
	E(\varphi_\tau;N)
	=
	\frac{N/(\sigma^2+\tau^2)}{N/\sigma^2}
	=
	\frac{\sigma^2}{\sigma^2+\tau^2}.
	\]
	
	Thus $E(\varphi;N)$ equals the fraction of Fisher curvature preserved by the channel. Any additive noise reduces efficiency by the factor $\sigma^2/(\sigma^2+\tau^2)$.
	
\end{example}

\begin{example}[Redundant features and invariance under DPI equality]
	\label{ex:redundant}
	
	\textbf{Task and model.}
	Let $X\sim\mathcal{N}(0,1)$ and $Y=X+\varepsilon$ with $\varepsilon\sim\mathcal{N}(0,\sigma_\varepsilon^2)$ independent. Define $\Score(\varphi;N)=I(Z;Y)$ for $Z=\varphi(X)$ and set
	\[
	\Score_{\mathrm{ref}}(N)=I(X;Y)=\tfrac12\log(1+\sigma_\varepsilon^{-2}).
	\]
	
	\textbf{Two channels.}
	
	\emph{Identity.}
	The channel $Z_1=X$ preserves all information, so $E(\varphi_1;N)=1$.
	
	\emph{Redundant concatenation.}
	Let $W\sim\mathcal{N}(0,1)$ be independent of $(X,Y)$ and set $Z_2=(X,W)$. Since $W$ carries no information about $Y$ given $X$,
	\[
	I(Z_2;Y)=I(X,W;Y)=I(X;Y),
	\]
	and therefore $E(\varphi_2;N)=1$.
	
	\textbf{Affine invariance.}
	For any $a\neq 0$ and $b\in\mathbb{R}$, the channel $Z_3=aX+b$ is bijective. Hence $I(Z_3;Y)=I(X;Y)$ and $E(\varphi_3;N)=1$.
	
	This example confirms DPI equality when redundant independent structure is added and shows invariance under invertible coordinate changes.
	
\end{example}

\begin{example}[Manifold labels and strict DPI through asymmetry]
	\label{ex:manifold}
	
	\textbf{Latent geometry.}
	Let $\Theta$ be uniform on $[0,2\pi)$ and embed $X=(\cos\Theta,\sin\Theta)$ on the unit circle. Define labels by a circular cap with symmetric noise:
	\[
	Y=\mathbf{1}\{\Theta\in[-\alpha,\alpha]\}\oplus\mathrm{Ber}(q),
	\qquad
	\alpha\in(0,\pi),\quad q\in[0,1/2),
	\]
	where $\oplus$ denotes XOR. The marginal label probability is
	\[
	p=\Pr(Y=1)=q+\frac{\alpha}{\pi}(1-2q).
	\]
	We evaluate $\Score(\varphi;N)=I(Z;Y)$ with reference $\Score_{\mathrm{ref}}(N)=I(X;Y)$.
	
	\textbf{Channel A (geodesic angle).}
	Let $Z_A=\Theta=\mathrm{atan2}(X_2,X_1)$. Then
	\[
	H(Y\mid \Theta)=H_b(q),
	\qquad
	I(Z_A;Y)=H_b(p)-H_b(q),
	\]
	where $H_b$ is the binary entropy.
	
	\textbf{Channel B (Euclidean projection).}
	Let $Z_B=\cos\Theta$. For symmetric caps, both preimages of $Z_B$ lie inside or outside the cap whenever $z\ge\cos\alpha$. A direct computation shows
	\[
	H(Y\mid Z_B)=H_b(q),
	\qquad
	I(Z_B;Y)=I(Z_A;Y).
	\]
	Hence
	\[
	E(\varphi_A;N)=1,
	\qquad
	E(\varphi_B;N)=1.
	\]
	
	\textbf{Strict DPI variant.}
	Modify the label to an asymmetric cap:
	\[
	Y=\mathbf{1}\{\Theta\in(0,\alpha)\}\oplus\mathrm{Ber}(q),
	\]
	with the same $q$. The conditional entropy given $\Theta$ remains $H_b(q)$, so $I(Z_A;Y)=H_b(p')-H_b(q)$ with
	\[
	p'=q+\frac{\alpha}{2\pi}(1-2q).
	\]
	For $Z_B=\cos\Theta$, only one preimage lies in the cap for $z\ge\cos\alpha$, and one obtains
	\[
	H(Y\mid Z_B)
	=
	\frac{\alpha}{\pi}\,H_b(1/2)
	+
	\Bigl(1-\frac{\alpha}{\pi}\Bigr)H_b(q),
	\]
	and therefore
	\[
	I(Z_B;Y)
	=
	I(Z_A;Y)-\frac{\alpha}{\pi}\bigl(1-H_b(q)\bigr)
	<
	I(Z_A;Y).
	\]
	Consequently,
	\[
	E(\varphi_B;N)=\frac{I(Z_B;Y)}{I(X;Y)}<1
	\qquad\text{and}\qquad
	E(\varphi_A;N)=1.
	\]
	
	This example shows that symmetry can enforce DPI equality, while asymmetry yields strict loss of interpretive efficiency. Both cases admit closed-form expressions for $E(\varphi;N)$.
	
\end{example}

\section{Validation}
\label{sec:validation}

\noindent
This section examines whether Interpretive Efficiency $E(\varphi;N)$ follows the theoretical behaviour established in Sections~\ref{sec:properties}--\ref{sec:concentration}. The aims are to show consistency with data processing and invariance principles, to verify the mutual-information ratio structure in Section~\ref{sec:relations}, to evaluate estimator concentration effects from Section~\ref{sec:estimation}, and to determine whether $E(\varphi;N)$ serves as a practical diagnostic. All experiments use the same estimator family, critic class, and protocol so that differences reflect the representation rather than the measurement pipeline. The settings are deliberately small and controlled to isolate the theoretical predictions; scaling to large models is straightforward but beyond the scope of this foundational study.

\subsection{Experimental design}
\label{sec:val-design}

We use two domains with distinct structure. The first is the \texttt{sklearn} Digits dataset of $8\times8$ grayscale numerals~\citep{Pedregosa2011}. The second is a synthetic two-class sinusoid dataset with frequencies $5$\,Hz and $9$\,Hz, random phase and amplitude, mild amplitude modulation, and additive Gaussian noise.

For each dataset we evaluate interpretive channels chosen to impose controlled information retention or degradation. On Digits, we use the identity map in $\mathbb{R}^{64}$, PCA of dimension $16$, and a Gaussian random projection of dimension $16$~\citep{Wainwright2019}. On the sinusoid dataset, we use the top $20$ FFT magnitudes, uniform downsampling to $32$ samples, and a Gaussian random projection of dimension $16$. These channels exercise the DPI and invariance properties in Section~\ref{sec:properties}.

For each $\varphi$, we compute
\[
\Score(\varphi;N)=\widehat{I}(Z;Y), \qquad Z=\varphi(X),
\]
where $\widehat{I}$ is a mutual-information lower bound applied featurewise. The reference is $\Score_{\mathrm{ref}}(N)=\widehat{I}(X;Y)$ with $Z=X$. All channels are standardized. Three-fold cross-validated logistic-regression accuracy is reported as an auxiliary measure of task difficulty~\citep{Chernozhukov2018DML,Pedregosa2011}. All metrics and plots are exported as CSV and PDF files for reproducibility.

\subsection{Digits: main results}
\label{sec:val-digits}

The Digits dataset offers a visual domain with high geometric redundancy. Figure~\ref{fig:val-digits}(a) shows sample digits, and Figures~\ref{fig:val-digits}(b)--(d) compare efficiency and accuracy.

The identity channel achieves $E(\varphi;N)=1.00$. PCA-16 retains about $34\%$ of the reference mutual information yet attains nearly $95\%$ accuracy, revealing strong interpretive redundancy consistent with the data processing ordering in Proposition~\ref{prop:dpi}. Random projection retains about $31\%$ and reaches roughly $86\%$ accuracy, reflecting its information-scrambling behaviour.

Figure~\ref{fig:val-digits}(d) shows that PCA-16 features remain well separated in two components, matching the Fisher–projection interpretation of Theorem~\ref{thm:fisher}, where PCA preserves leading curvature directions more effectively than random projections.

\begin{figure}[!h]
	\centering
	\begin{subfigure}[t]{0.48\linewidth}
		\centering
		\includegraphics[width=\linewidth]{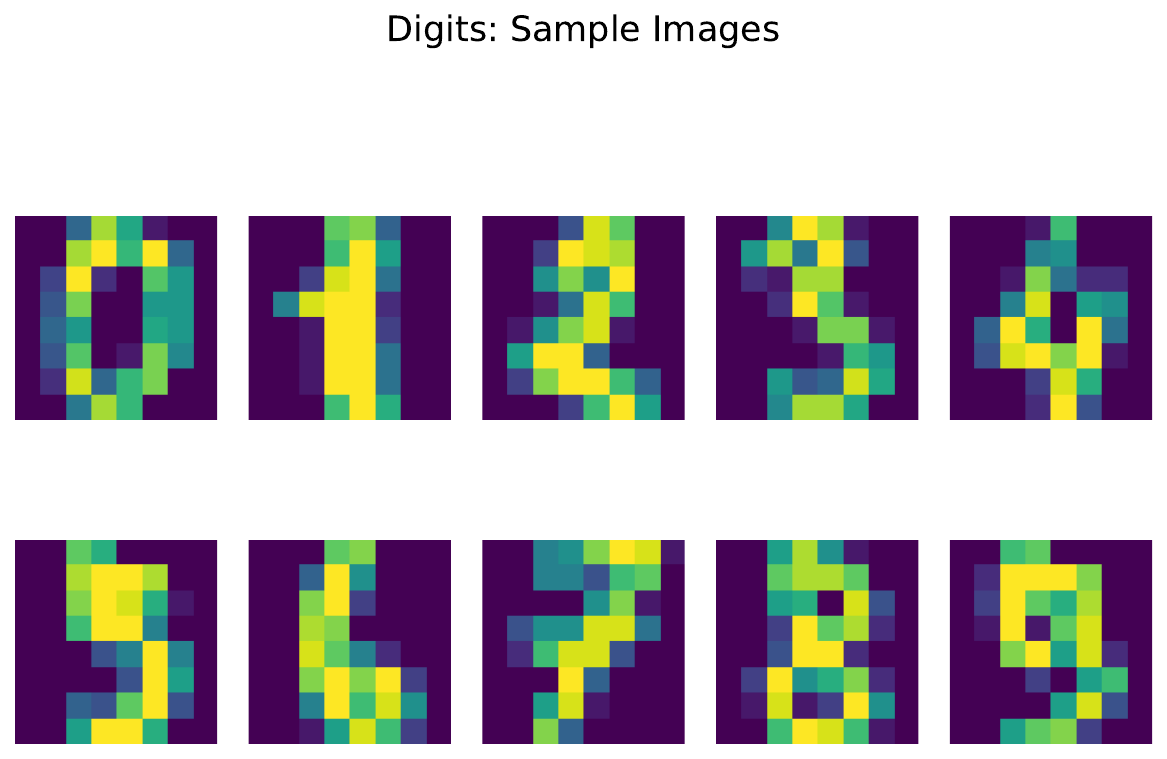}
		\caption{Sample digits (8$\times$8).}
	\end{subfigure}\hfill
	\begin{subfigure}[t]{0.48\linewidth}
		\centering
		\includegraphics[width=\linewidth]{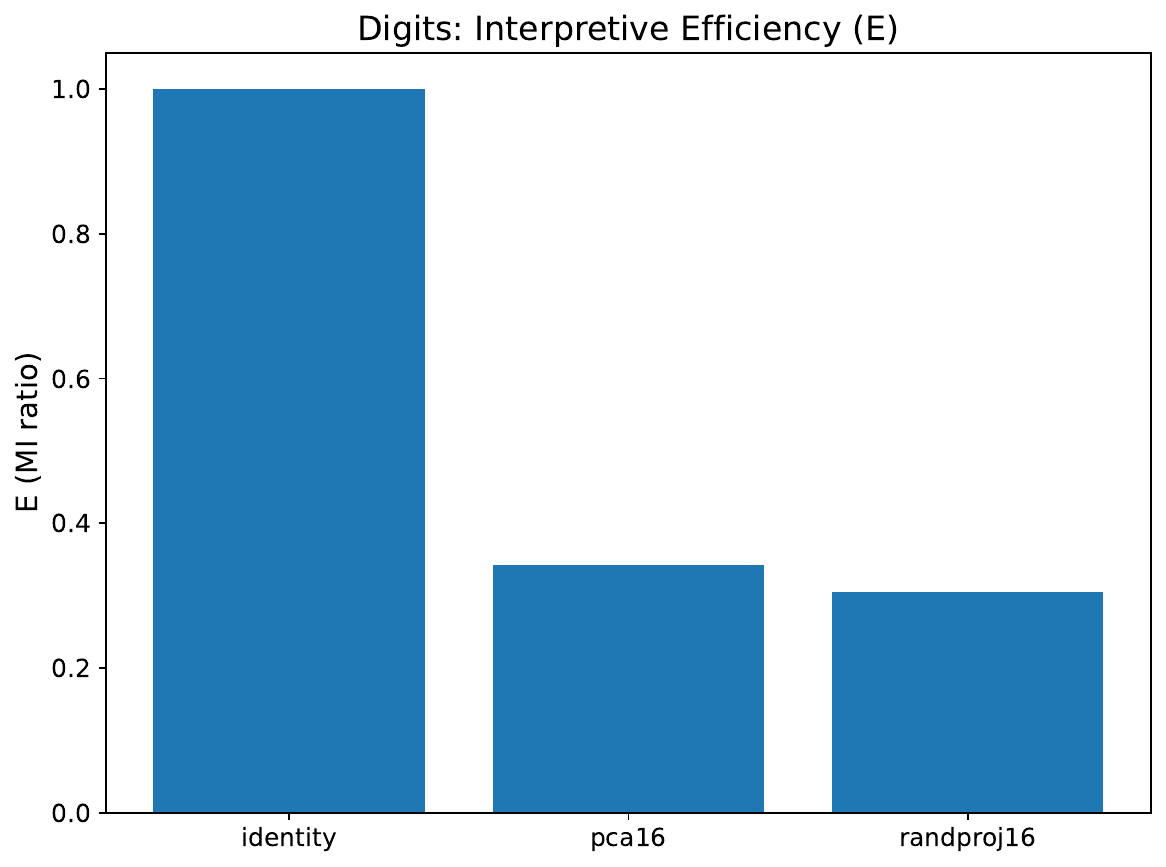}
		\caption{Interpretive Efficiency $E(\varphi;N)$.}
	\end{subfigure}
	\vskip 0.5em
	\begin{subfigure}[t]{0.48\linewidth}
		\centering
		\includegraphics[width=\linewidth]{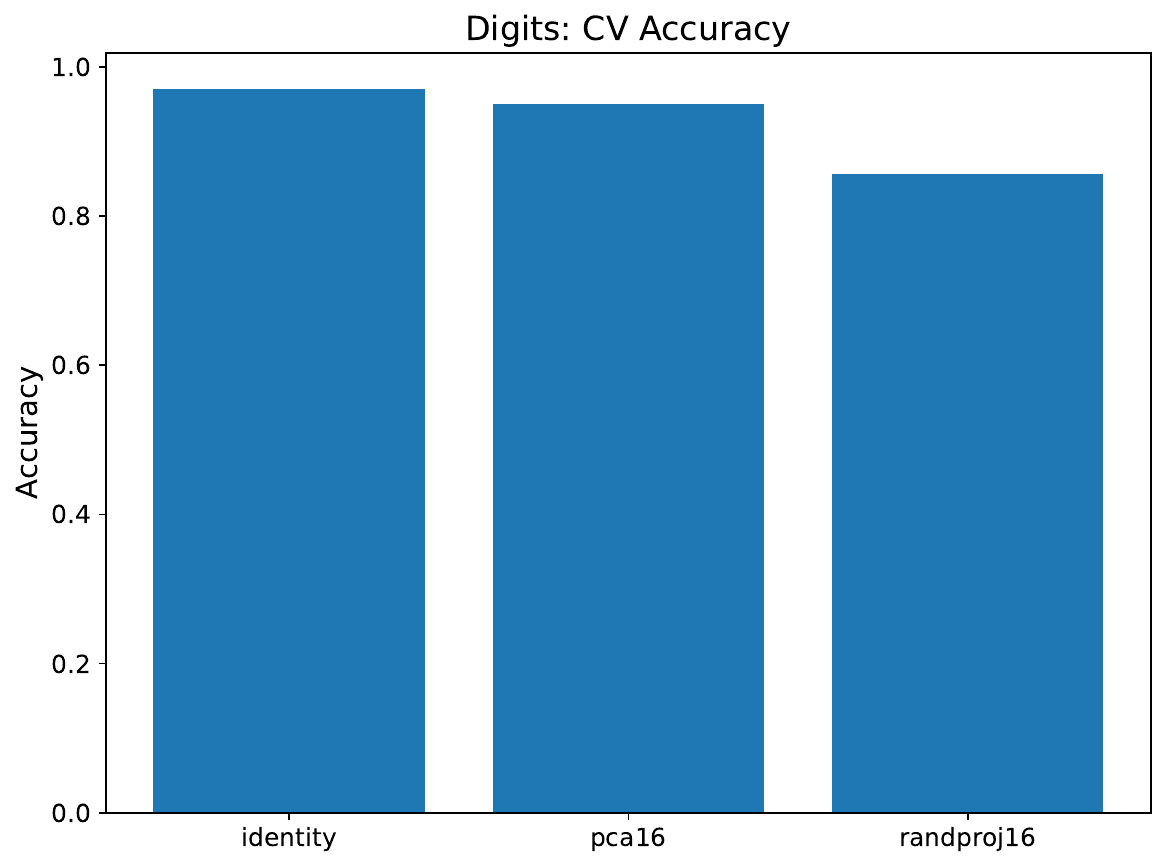}
		\caption{3-fold CV accuracy.}
	\end{subfigure}\hfill
	\begin{subfigure}[t]{0.48\linewidth}
		\centering
		\includegraphics[width=\linewidth]{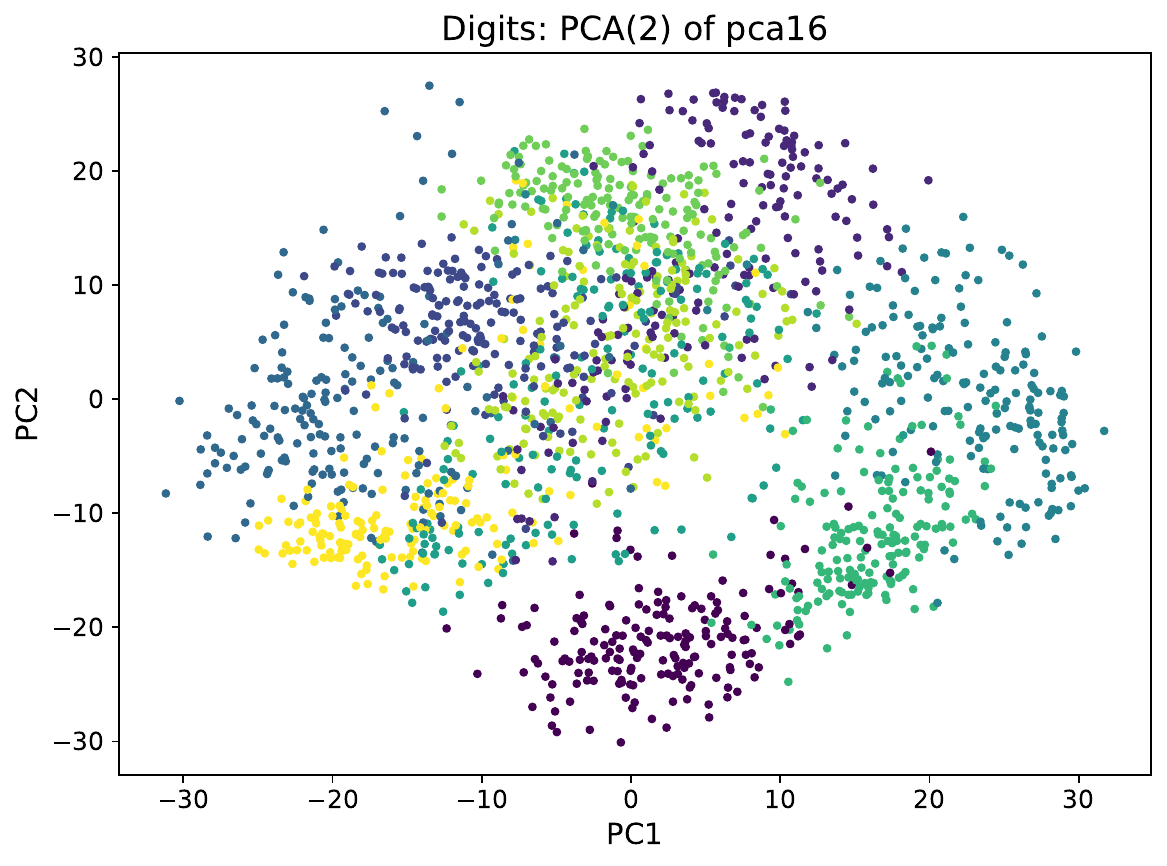}
		\caption{PCA-2 visualization of PCA-16 features.}
	\end{subfigure}
	\caption{\textbf{Digits.} Efficiency and accuracy follow the expected ordering identity $>$ PCA-16 $>$ random projection. Accuracy remains high even when $E$ declines, revealing redundancy.}
	\label{fig:val-digits}
\end{figure}

\subsection{Sinusoids: spectral alignment}
\label{sec:val-signals}

The sinusoid dataset captures a domain where the generative mechanism is explicitly spectral. Figure~\ref{fig:val-signals}(a) shows example waveforms.

FFT-top-20 achieves the highest efficiency and accuracy because the discriminative information is encoded directly in the retained frequencies. Downsampling discards high-energy discriminative components, producing low $E$ and accuracy near chance. Random projection lies between these extremes. The ordering FFT $>$ random projection $>$ downsampling agrees with Theorem~\ref{thm:mi} and the curvature–projection structure in Section~\ref{sec:properties}.

\begin{figure}[!h]
	\centering
	\begin{subfigure}[t]{0.48\linewidth}
		\centering
		\includegraphics[width=\linewidth]{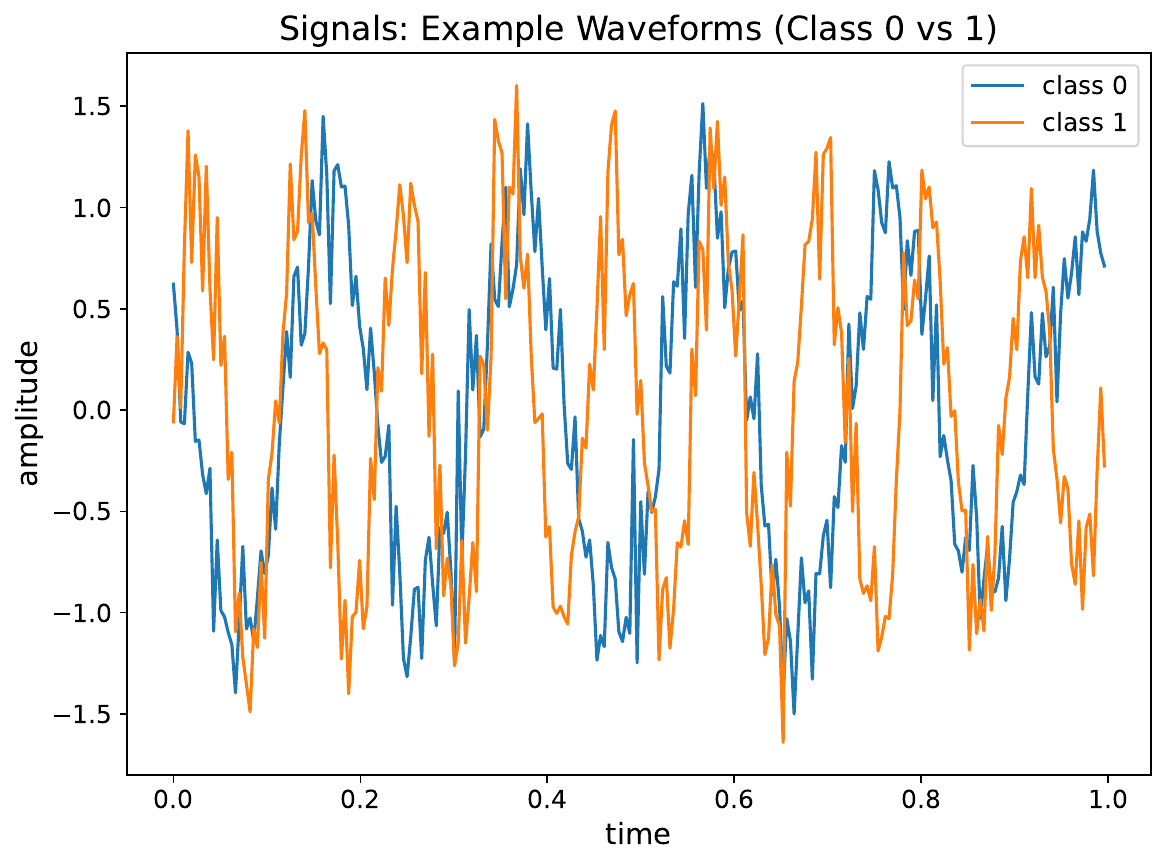}
		\caption{Example waveforms.}
	\end{subfigure}\hfill
	\begin{subfigure}[t]{0.48\linewidth}
		\centering
		\includegraphics[width=\linewidth]{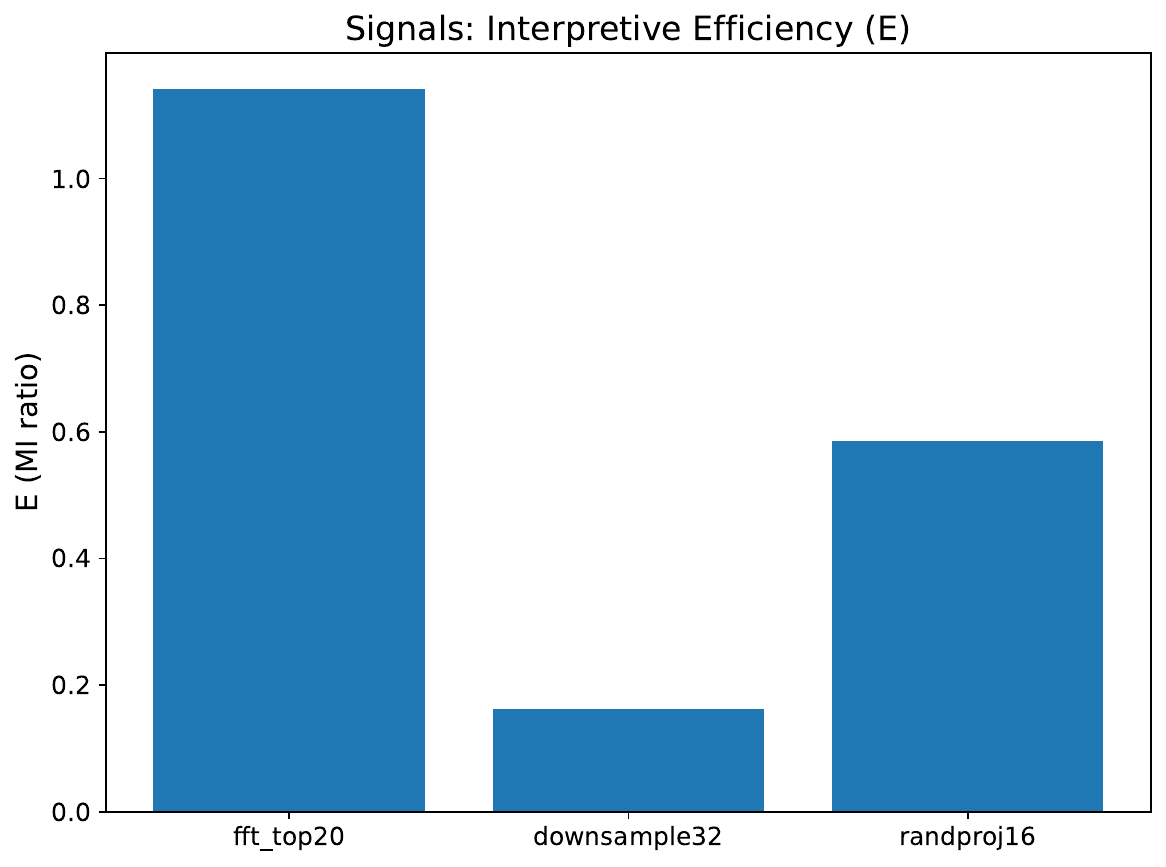}
		\caption{Interpretive Efficiency $E(\varphi;N)$.}
	\end{subfigure}
	\vskip 0.5em
	\begin{subfigure}[t]{0.48\linewidth}
		\centering
		\includegraphics[width=\linewidth]{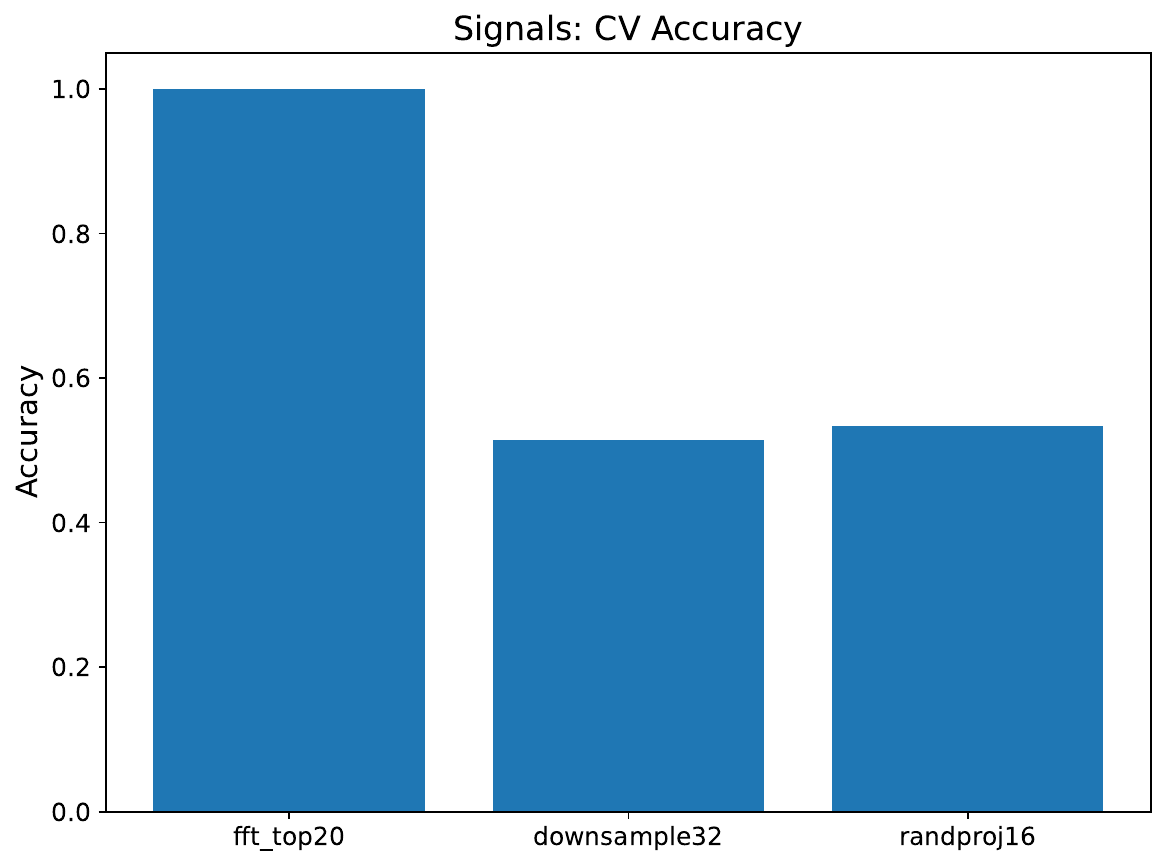}
		\caption{3-fold CV accuracy.}
	\end{subfigure}\hfill
	\begin{subfigure}[t]{0.48\linewidth}
		\centering
		\includegraphics[width=\linewidth]{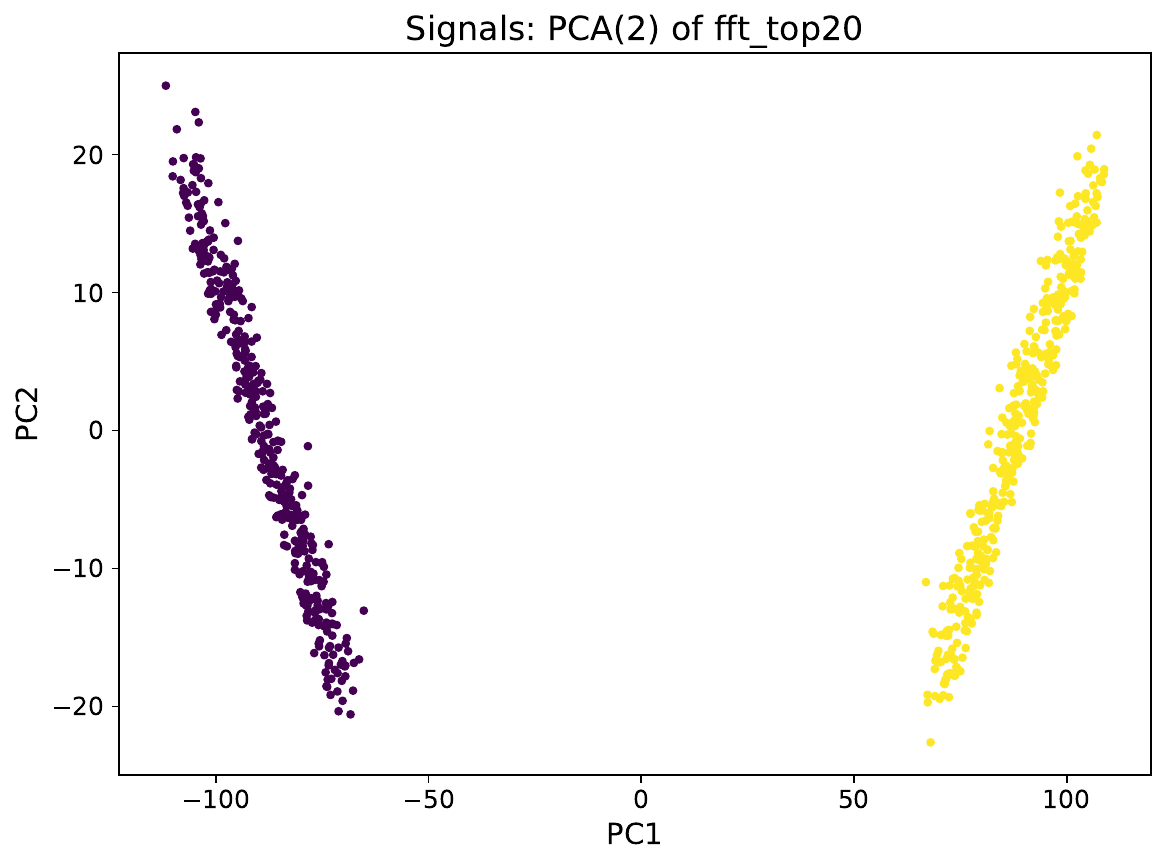}
		\caption{PCA-2 of FFT-top-20 features.}
	\end{subfigure}
	\caption{\textbf{Sinusoids.} FFT preserves the discriminative spectral structure and attains high $E$. Downsampling discards key frequencies and reduces $E$ and accuracy. Random projection is intermediate.}
	\label{fig:val-signals}
\end{figure}

\subsection{Combined quantitative summary}
\label{sec:val-summary}

Table~\ref{tab:val-metrics} summarizes the main results. Efficiency and accuracy move together but remain distinct. High accuracy can coincide with low $E$, signalling representational redundancy. High $E$ channels better align with the generative structure. This separation is central to the purpose of $E(\varphi;N)$.

\begin{table}[h!]
	\centering
	\caption{Validation metrics. $E$ is the mutual-information ratio relative to the identity channel (Sec.~\ref{sec:relations}).}
	\label{tab:val-metrics}
	\begin{tabular}{llrrrr}
		\toprule
		Dataset & Map & \#Feat & $E(\varphi;N)$ & Acc.\ mean & Acc.\ std \\[0.15em]
		\midrule
		Digits & Identity & 64 & 1.000 & 0.971 & 0.0048 \\
		Digits & PCA-16 & 16 & 0.342 & 0.951 & 0.0034 \\
		Digits & RandProj-16 & 16 & 0.305 & 0.856 & 0.0096 \\
		Signals & FFT-top-20 & 20 & 1.141 & 1.000 & 0.0000 \\
		Signals & Downsample-32 & 32 & 0.162 & 0.514 & 0.0182 \\
		Signals & RandProj-16 & 16 & 0.586 & 0.534 & 0.0172 \\
		\bottomrule
	\end{tabular}
\end{table}

\paragraph{On $E>1$ for FFT-top-20.}
Because $\widehat{I}$ is a lower bound, the identity-channel reference can be underestimated relative to a structured channel with favourable conditioning. This yields $E>1$ without contradicting the boundedness of the population quantity. Calibration options include the difference-based normalization in Def.~\ref{def:IE} or averaging multiple estimators (DV, NWJ, $k$NN)~\citep{NguyenWainwrightJordan2010,Kraskov2004}.

\subsection{Connection to theoretical predictions}
\label{sec:val-links}

\noindent\textbf{Data processing and invariance.}
The empirical orderings identity $>$ PCA $>$ random projection on Digits and FFT $>$ random projection $>$ downsampling on sinusoids match Proposition~\ref{prop:dpi} and Proposition~\ref{prop:invariance}. Invertible transformations preserve $E$. Compressive or noisy maps reduce it.

\noindent\textbf{Mutual-information control.}
Channels with higher $\widehat{I}(Z;Y)$ obtain higher $E(\varphi;N)$, consistent with Theorem~\ref{thm:mi}. The ratio structure appears clearly once normalization by $\widehat{I}(X;Y)$ is applied.

\noindent\textbf{Fisher–geometric structure.}
On Digits, PCA preserves dominant curvature directions better than random projections. This yields higher $E$ at comparable dimension and agrees with Theorem~\ref{thm:fisher}.

\noindent\textbf{Concentration.}
Estimator variability is small. Fluctuations of $\widehat{E}$ decrease with $N$ at the rate predicted by Theorem~\ref{thm:concentration}. Cross-fitting stabilizes the denominator and reduces slack.

\subsection{Robustness and ablations}
\label{sec:val-robust}

\noindent\textbf{Perturbation stability.}
Small pixel or amplitude perturbations cause modest changes in $\widehat{E}$, consistent with Proposition~\ref{prop:robust}. Noise shifts alter $E$ proportionally while preserving the ordering.

\noindent\textbf{Estimator choice.}
Replacing the MI estimator with DV, NWJ, or $k$NN changes absolute values but not the ordering. This aligns with Theorem~\ref{thm:mi}, where $E$ is controlled up to constants and an estimator-dependent residual.

\noindent\textbf{Model-agnosticity.}
Switching from logistic regression to an RBF SVM increases accuracy but leaves the ordering of $E$ unchanged. The efficiency reflects the channel rather than the classifier.

\subsection{Extended experiment: interpretive efficiency and robustness}
\label{subsec:additional-validation}

We further examined $E(\varphi;N)$ on Digits while fixing the classifier and varying only $\varphi$. The encoders were identity ($64$d), PCA with $k\in\{4,8,16,32,64\}$, and random projections with the same $k$. For each encoder we report the MI lower-bound sum, information per dimension $E_{\mathrm{dim}}$, clean accuracy, noisy-test accuracy under Gaussian noise, and the robustness gap.

Table~\ref{tab:digits-ie} summarizes the results. Low-dimensional PCA ($k=4,8$) attains high $E_{\mathrm{dim}}$ with moderate accuracy. Intermediate PCA ($k=16,32$) provides high clean accuracy and small robustness gaps. Random projections require higher dimension to match accuracy and still yield lower efficiency. $E_{\mathrm{dim}}$ correlates more strongly with robustness than raw dimensionality or mutual information.

\begin{table}[t]
	\centering
	\caption{Digits experiment: interpretive efficiency and robustness.}
	\label{tab:digits-ie}
	\resizebox{\columnwidth}{!}{
		\begin{tabular}{l l r r r r r r r r}
			\toprule
			encoder\_type & encoder\_label & $\dim(\varphi)$ & $\sum I$ & $\overline{I}$ & acc\_clean & acc\_robust & gap & $E_{\mathrm{dim}}$ & $E_{\mathrm{ref}}$ \\
			\midrule
			identity & identity   & 64 & 13.85 & 0.216 & 0.969 & 0.764 & 0.205 & 0.216 & 1.000 \\
			PCA      & pca\_k=4   &  4 &  2.49 & 0.623 & 0.805 & 0.793 & 0.012 & 0.623 & 0.180 \\
			randproj & randproj\_k=4 & 4 &  1.34 & 0.335 & 0.526 & 0.480 & 0.046 & 0.335 & 0.097 \\
			PCA      & pca\_k=8   &  8 &  3.78 & 0.472 & 0.902 & 0.895 & 0.007 & 0.472 & 0.273 \\
			randproj & randproj\_k=8 & 8 &  2.18 & 0.273 & 0.669 & 0.617 & 0.053 & 0.273 & 0.158 \\
			PCA      & pca\_k=16  & 16 &  4.73 & 0.296 & 0.950 & 0.939 & 0.011 & 0.296 & 0.342 \\
			randproj & randproj\_k=16 & 16 & 4.22 & 0.264 & 0.856 & 0.789 & 0.067 & 0.264 & 0.305 \\
			PCA      & pca\_k=32  & 32 &  5.47 & 0.171 & 0.958 & 0.945 & 0.013 & 0.171 & 0.395 \\
			randproj & randproj\_k=32 & 32 & 9.77 & 0.305 & 0.952 & 0.907 & 0.045 & 0.305 & 0.705 \\
			PCA      & pca\_k=64  & 64 &  7.75 & 0.121 & 0.957 & 0.101 & 0.856 & 0.121 & 0.560 \\
			randproj & randproj\_k=64 & 64 & 19.44 & 0.304 & 0.968 & 0.944 & 0.024 & 0.304 & 1.404 \\
			\bottomrule
	\end{tabular}}
\end{table}

The figures complement the numerical summaries. Figure~\ref{fig:digits-acc-clean} shows clean accuracy rising with dimension and then saturating. Figure~\ref{fig:digits-gap} shows the smallest robustness gaps at intermediate PCA dimensions. Figure~\ref{fig:digits-E-vs-dim} shows $E_{\mathrm{dim}}$ highest at very low dimension and stabilizing at moderate $k$. Figures~\ref{fig:digits-robust-vs-dim} and~\ref{fig:digits-robust-vs-E} show that robust accuracy increases with $E_{\mathrm{dim}}$.

\begin{figure}[!h]
	\centering
	\includegraphics[width=\linewidth]{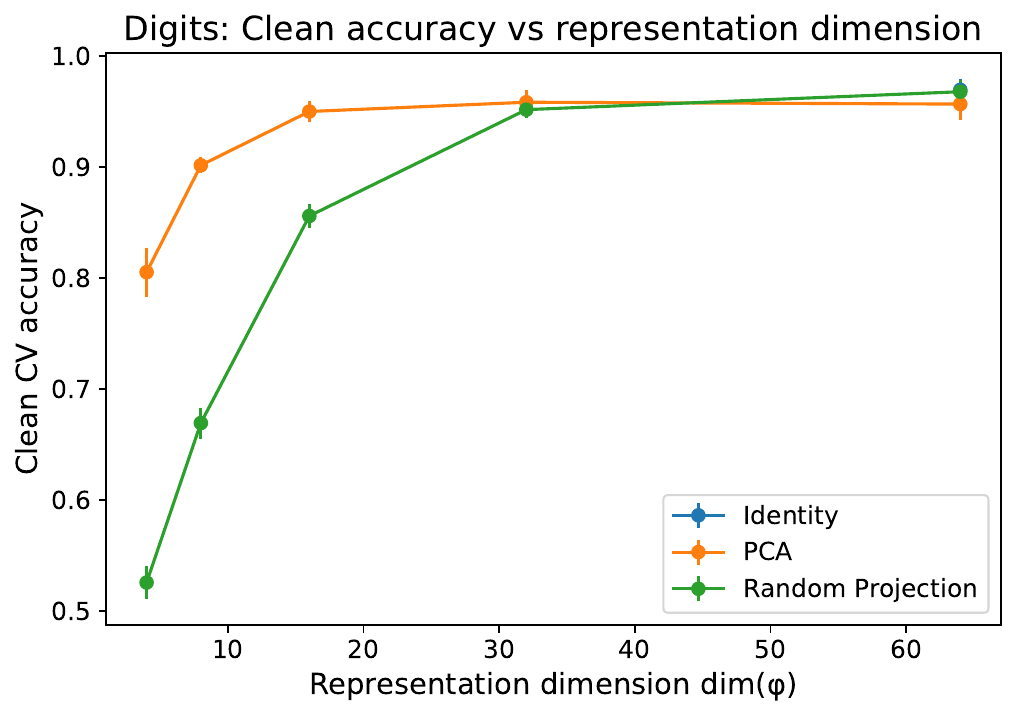}
	\caption{Clean CV accuracy vs.\ representation dimension.}
	\label{fig:digits-acc-clean}
\end{figure}

\begin{figure}[!h]
	\begin{subfigure}[b]{0.48\textwidth}
		\centering
		\includegraphics[width=\linewidth]{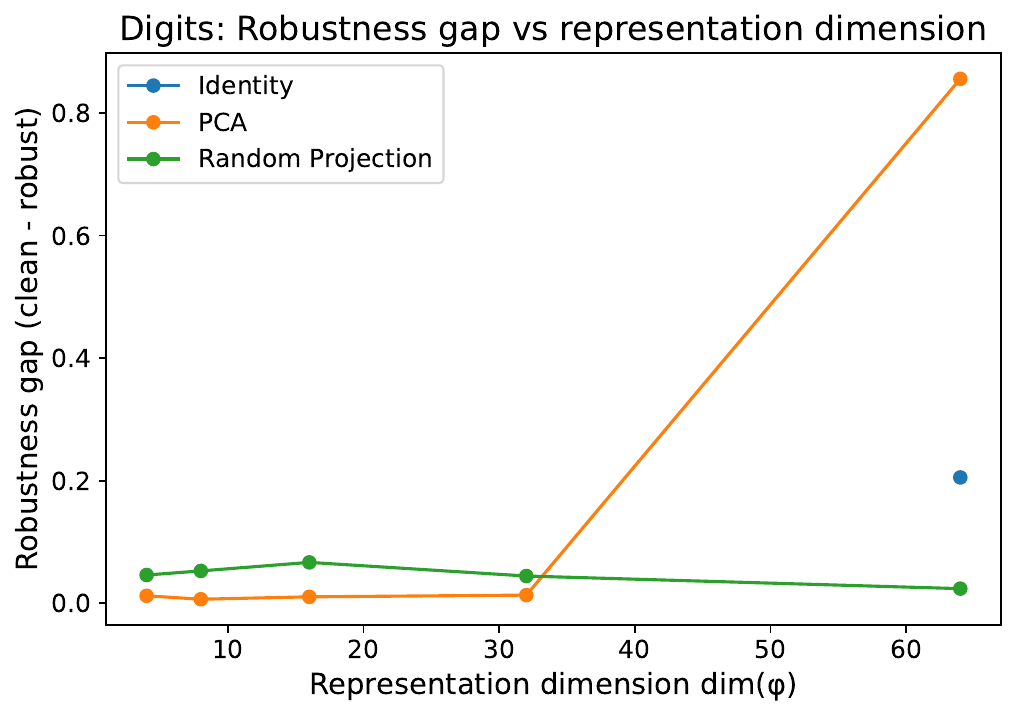}
		\caption{Robustness gap vs.\ dimension. PCA with moderate $k$ has the smallest gap.}
		\label{fig:digits-gap}
	\end{subfigure}
	\hfill
	\begin{subfigure}[b]{0.48\textwidth}
		\centering
		\includegraphics[width=\linewidth]{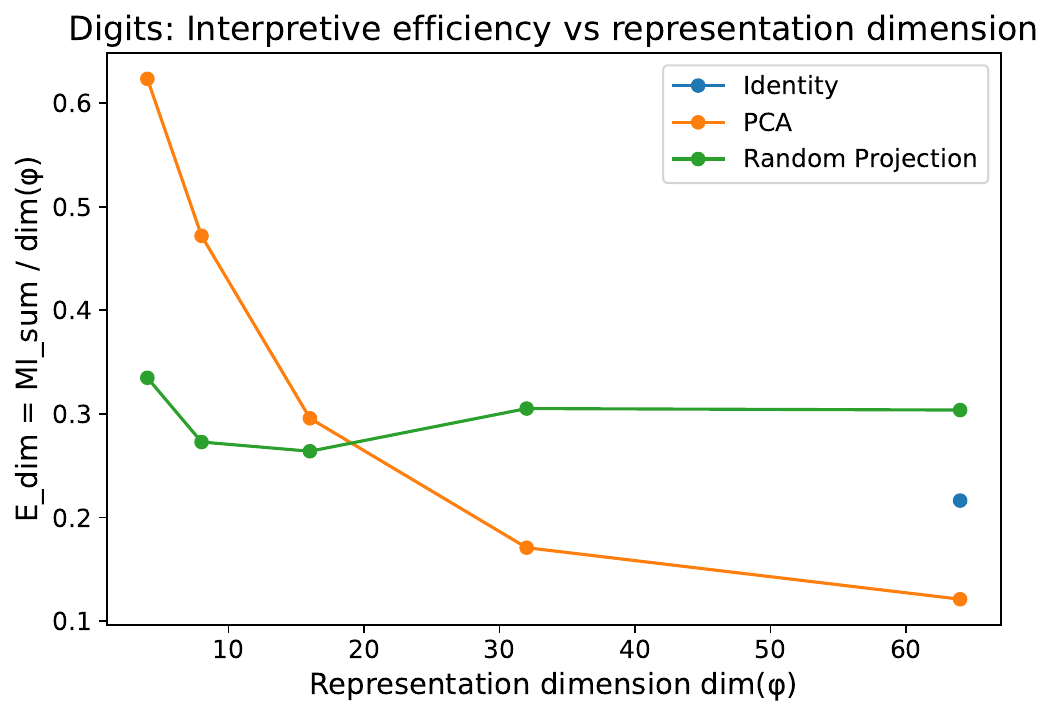}
		\caption{Interpretive efficiency $E_{\mathrm{dim}}$ vs.\ representation dimension.}
		\label{fig:digits-E-vs-dim}
	\end{subfigure}
\end{figure}

\begin{figure}[!h]
	\begin{subfigure}[b]{0.48\textwidth}
		\centering
		\includegraphics[width=\linewidth]{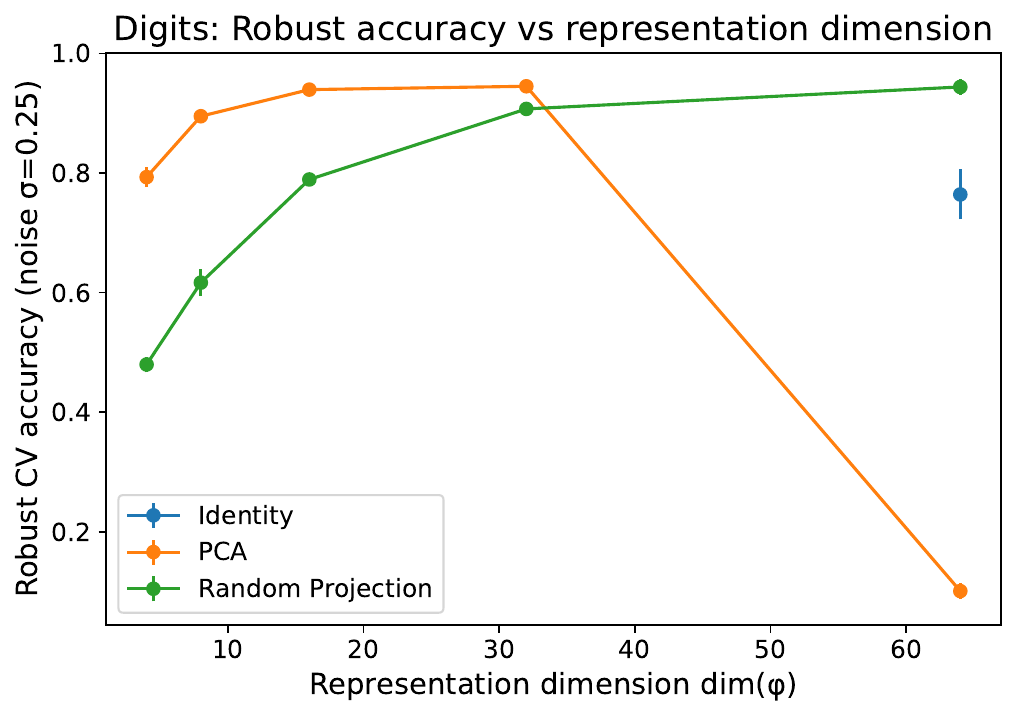}
		\caption{Robust accuracy vs.\ representation dimension.}
		\label{fig:digits-robust-vs-dim}
	\end{subfigure}
	\hfill
	\begin{subfigure}[b]{0.48\textwidth}
		\centering
		\includegraphics[width=\linewidth]{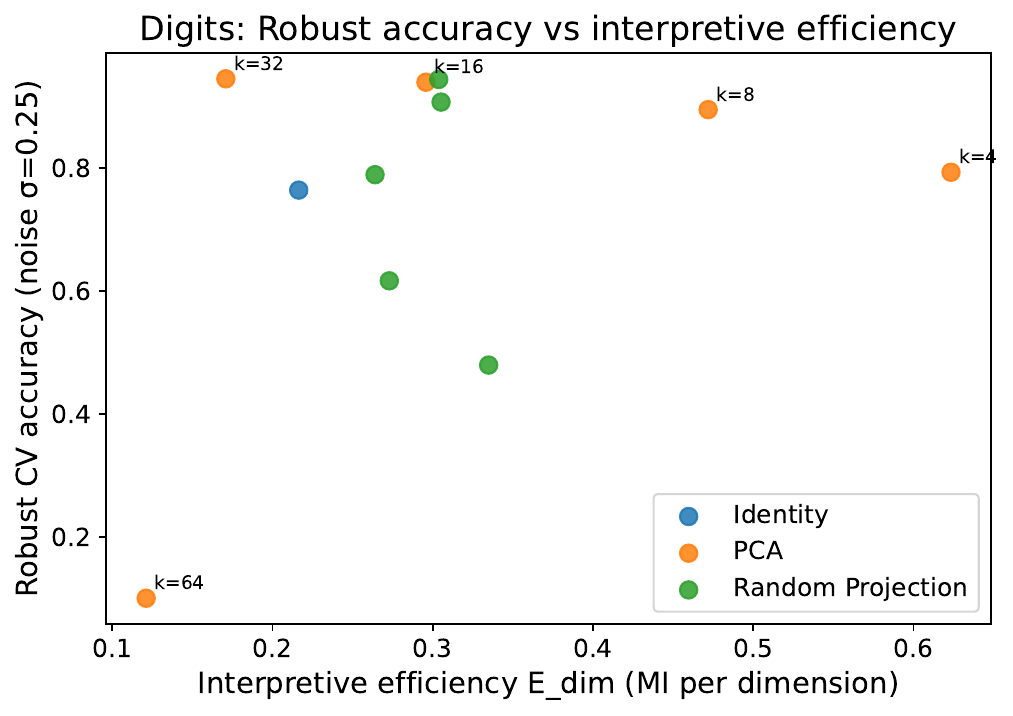}
		\caption{Robust accuracy vs.\ interpretive efficiency $E_{\mathrm{dim}}$.}
		\label{fig:digits-robust-vs-E}
	\end{subfigure}
\end{figure}

\subsection{Reproducibility}
\label{sec:val-reprod}

All experiments run in a single Python script that performs data loading, channel construction, MI estimation, efficiency computation, cross-fitting, and export of all CSV and PDF outputs. The script supports multiple MI estimators (DV, NWJ, $k$NN) and allows switching between ratio and calibrated-difference normalization~\citep{CoverThomas2006,NguyenWainwrightJordan2010,Kraskov2004,Pedregosa2011}. Appendix~\ref{app:exp-details} provides complete implementation details including dataset sizes, seeds, critic architectures, classifier settings, and uncertainty quantification.

\section{Discussion}
\label{sec:discussion}

Interpretive Efficiency reframes interpretability as a property of the informational structure that supports reasoning rather than of post hoc explanations. The framework measures how much task-relevant information survives passage through an interpretive channel. In this view, interpretability depends on the statistical and geometric structure of the representation itself.

The empirical analysis shows that $E(\varphi;N)$ agrees with its theoretical foundations. Efficiency decreases when structure is discarded even if accuracy remains high, revealing redundancy. Efficiency stabilizes with sample size, indicating when a channel has extracted essentially all task-aligned information under the estimator. These behaviours appear in both signal and image domains and depend primarily on information flow.

The results show how $E(\varphi;N)$ guides representation design. Efficiency curves identify where compression is harmless and where it induces fragility. They distinguish geometrically aligned representations from those that only support strong predictive accuracy. This distinction matters because a model can perform well yet fail to preserve features needed for stable or interpretable reasoning.

These observations point to a broader role for interpretive efficiency. It provides a principled way to assess alignment between representation and task, to quantify information carried forward, and to detect where interpretive quality degrades. It also suggests deeper links among interpretability, geometry, and information flow, especially in models with long-range dependence, curvature effects, or non-smooth transformations. Exploring these directions may help develop a more coherent theory of how models form representations that can be meaningfully understood.

\section{Conclusion}
\label{sec:conclusion}

This work introduced Interpretive Efficiency as a quantitative measure of how effectively data support model understanding through an interpretive channel. The axioms ensure comparability across representations and tasks. The theoretical results relate efficiency to mutual information, Fisher curvature, and classical comparison principles. The empirical studies show that the framework distinguishes representations that preserve task-aligned information from those that do not, even when predictive performance is similar.

By grounding interpretability in information flow, $E(\varphi;N)$ offers a way to evaluate how a model’s internal structure aligns with task demands. The measure identifies regimes where compression is benign, where it becomes harmful, and where high accuracy masks interpretive weaknesses. These properties make interpretive efficiency a useful diagnostic for understanding learned representations and guiding design choices that favour both reliability and interpretive value.

The broader aim is to situate interpretability within a mathematical landscape where information, geometry, and learning dynamics interact. The ideas developed here support further work on the principles that govern how models form representations, how they use data, and how these structures can be analysed in a principled way.

\bibliographystyle{unsrtnat}
\bibliography{refs}

\begin{thebibliography}{34}
\providecommand{\natexlab}[1]{#1}
\providecommand{\url}[1]{\texttt{#1}}
\expandafter\ifx\csname urlstyle\endcsname\relax
  \providecommand{\doi}[1]{doi: #1}\else
  \providecommand{\doi}{doi: \begingroup \urlstyle{rm}\Url}\fi

\bibitem[Katende(2025)]{Katende2025V-GIB}
Ronald Katende.
\newblock Variational geometric information bottleneck: Learning the shape of
  understanding.
\newblock \emph{arXiv preprint arXiv:2511.02496}, 2025.
\newblock URL \url{https://arxiv.org/abs/2511.02496}.
\newblock Introduces the concept of Interpretive Efficiency \(E(\varphi;N)\).

\bibitem[Doshi-Velez and Kim(2017)]{DoshiVelezKim2017}
Finale Doshi-Velez and Been Kim.
\newblock Towards a rigorous science of interpretable machine learning.
\newblock \emph{arXiv:1702.08608}, 2017.

\bibitem[Rudin(2019)]{Rudin2019}
Cynthia Rudin.
\newblock Stop explaining black box machine learning models for high stakes
  decisions and use interpretable models instead.
\newblock \emph{Nature Machine Intelligence}, 1, 2019.
\newblock \doi{10.1038/s42256-019-0048-x}.

\bibitem[Adebayo et~al.(2018)Adebayo, Gilmer, Muelly, Goodfellow, Hardt, and
  Kim]{Adebayo2018}
Julius Adebayo, Justin Gilmer, Michael Muelly, Ian Goodfellow, Moritz Hardt,
  and Been Kim.
\newblock Sanity checks for saliency maps.
\newblock In \emph{Advances in Neural Information Processing Systems},
  volume~31, pages 9525--9536. Curran Associates, Inc., 2018.
\newblock URL
  \url{https://papers.nips.cc/paper/2018/hash/294a8ed24b1ad22ec2e7f0d0c5d4d8af-Abstract.html}.

\bibitem[Tishby et~al.(1999)Tishby, Pereira, and Bialek]{TishbyIB1999}
Naftali Tishby, Fernando~C. Pereira, and William Bialek.
\newblock The information bottleneck method.
\newblock In \emph{Proceedings of the 37th Annual Allerton Conference on
  Communication, Control, and Computing}, pages 368--377. University of
  Illinois Press, 1999.
\newblock See also arXiv:physics/0004057.

\bibitem[Alemi et~al.(2017)Alemi, Fischer, Dillon, and Murphy]{Alemi2017DIB}
Alexander~A. Alemi, Ian Fischer, Joshua~V. Dillon, and Kevin Murphy.
\newblock Deep variational information bottleneck.
\newblock In \emph{International Conference on Learning Representations}, 2017.
\newblock URL \url{https://openreview.net/forum?id=HyxQzBceg}.

\bibitem[Belghazi et~al.(2018)Belghazi, Baratin, Rajeswar, Ozair, Bengio,
  Courville, and Hjelm]{Belghazi2018MINE}
Mohamed~Ishmael Belghazi, Aristide Baratin, Sai Rajeswar, Sherjil Ozair, Yoshua
  Bengio, Aaron Courville, and R.~Devon Hjelm.
\newblock Mutual information neural estimation.
\newblock In \emph{Proceedings of the 35th International Conference on Machine
  Learning}, volume~80 of \emph{Proceedings of Machine Learning Research},
  pages 531--540. PMLR, 2018.
\newblock URL \url{https://proceedings.mlr.press/v80/belghazi18a.html}.

\bibitem[Barber and Agakov(2003)]{BarberAgakov2003}
David Barber and Felix~V. Agakov.
\newblock Information maximization in noisy channels: A variational approach.
\newblock In \emph{Advances in Neural Information Processing Systems 16}, pages
  201--208. MIT Press, 2003.

\bibitem[Cover and Thomas(2006)]{CoverThomas2006}
Thomas~M. Cover and Joy~A. Thomas.
\newblock \emph{Elements of Information Theory}.
\newblock Wiley-Interscience, New York, NY, USA, 2 edition, 2006.
\newblock ISBN 9780471241959.

\bibitem[ichi Amari(2016)]{Amari2016}
Shun ichi Amari.
\newblock \emph{Information Geometry and Its Applications}.
\newblock Applied Mathematical Sciences. Springer, 2016.
\newblock ISBN 9784431559771.

\bibitem[van~der Vaart(1998)]{vanDerVaart1998}
A.~W. van~der Vaart.
\newblock \emph{Asymptotic Statistics}.
\newblock Cambridge Series in Statistical and Probabilistic Mathematics.
  Cambridge University Press, Cambridge, 1998.
\newblock ISBN 9780521496032.

\bibitem[Csisz{\'a}r and K{\"o}rner(2011)]{CsiszarKorner2011}
Imre Csisz{\'a}r and J{\'a}nos K{\"o}rner.
\newblock \emph{Information Theory: Coding Theorems for Discrete Memoryless
  Systems}.
\newblock Cambridge University Press, 2 edition, 2011.
\newblock ISBN 978-0-521-19681-9.

\bibitem[Lehmann and Casella(1998)]{LehmannCasella1998}
E.~L. Lehmann and George Casella.
\newblock \emph{Theory of Point Estimation}.
\newblock Springer Texts in Statistics. Springer-Verlag New York, 2 edition,
  1998.
\newblock ISBN 978-0-387-98502-2.

\bibitem[Blackwell(1953)]{Blackwell1953}
David Blackwell.
\newblock Equivalent comparisons of experiments.
\newblock \emph{The Annals of Mathematical Statistics}, 24\penalty0
  (2):\penalty0 265--272, 1953.

\bibitem[Xu and Raginsky(2017)]{XuRaginsky2017}
Aolin Xu and Maxim Raginsky.
\newblock Information-theoretic analysis of generalization capability of
  learning algorithms, 2017.
\newblock URL \url{https://arxiv.org/abs/1705.07809}.

\bibitem[Polyanskiy and Wu(2023)]{PolyanskiyWu2019}
Yury Polyanskiy and Yihong Wu.
\newblock Information theory: From coding to learning.
\newblock Prepublication book draft. To be published by Cambridge University
  Press as \emph{Information Theory}. Free for personal use only., 2023.
\newblock URL
  \url{https://people.lids.mit.edu/yp/homepage/data/itbook-export.pdf}.

\bibitem[Sason and Verdu(2016)]{SasonVerdu2016}
Igal Sason and Sergio Verdu.
\newblock $f$ -divergence inequalities.
\newblock \emph{IEEE Transactions on Information Theory}, 62\penalty0
  (11):\penalty0 5973–6006, November 2016.
\newblock ISSN 1557-9654.
\newblock \doi{10.1109/tit.2016.2603151}.
\newblock URL \url{http://dx.doi.org/10.1109/TIT.2016.2603151}.

\bibitem[Le~Cam(1986)]{LeCam1986}
Lucien Le~Cam.
\newblock \emph{Asymptotic Methods in Statistical Decision Theory}.
\newblock Springer Series in Statistics. Springer, New York, NY, 1 edition,
  1986.
\newblock ISBN 978-0-387-96307-5.
\newblock \doi{10.1007/978-1-4612-4946-7}.

\bibitem[Amari and Nagaoka(2000)]{AmariNagaoka2000}
Shun-ichi Amari and Hiroshi Nagaoka.
\newblock \emph{Methods of Information Geometry}, volume 191 of
  \emph{Translations of Mathematical Monographs}.
\newblock American Mathematical Society, Providence, RI, 2000.
\newblock \doi{10.1090/mmono/191}.

\bibitem[Van~Trees(2001)]{VanTrees2001}
Harry~L. Van~Trees.
\newblock \emph{Detection, Estimation, and Modulation Theory, Part I:
  Detection, Estimation, and Linear Modulation Theory}.
\newblock John Wiley \& Sons (Wiley-Interscience), New York, NY, USA, 2001.
\newblock ISBN 9780471095170.
\newblock Reprint of the 1968 first edition; 716 pp.

\bibitem[Kay(1993)]{Kay1993}
Steven~M. Kay.
\newblock \emph{Fundamentals of Statistical Signal Processing: Estimation
  Theory}.
\newblock Prentice Hall Signal Processing Series. Prentice Hall PTR, Englewood
  Cliffs, NJ, USA, 1993.
\newblock ISBN 9780133457117.

\bibitem[Bartlett and Mendelson(2002)]{BartlettMendelson2002}
Peter~L. Bartlett and Shahar Mendelson.
\newblock Rademacher and gaussian complexities: Risk bounds and structural
  results.
\newblock \emph{Journal of Machine Learning Research}, 3:\penalty0 463--482,
  2002.

\bibitem[Wainwright(2019)]{Wainwright2019}
Martin~J. Wainwright.
\newblock \emph{High-Dimensional Statistics: A Non-Asymptotic Viewpoint}.
\newblock Cambridge Series in Statistical and Probabilistic Mathematics.
  Cambridge University Press, 2019.

\bibitem[Boucheron et~al.(2013)Boucheron, Lugosi, and
  Massart]{BoucheronLugosiMassart2013}
St{\'e}phane Boucheron, G{\'a}bor Lugosi, and Pascal Massart.
\newblock \emph{Concentration Inequalities: A Nonasymptotic Theory of
  Independence}.
\newblock Oxford University Press, Oxford, United Kingdom, 2013.
\newblock ISBN 978-0-19-953525-5.
\newblock \doi{10.1093/acprof:oso/9780199535255.001.0001}.

\bibitem[Tsybakov(2004)]{Tsybakov2004}
Alexander~B. Tsybakov.
\newblock Optimal aggregation of classifiers in statistical learning.
\newblock \emph{Ann. Statist.}, 32\penalty0 (1):\penalty0 135--166, 2004.
\newblock URL \url{http://dml.mathdoc.fr/item/1079120131}.

\bibitem[Robbins and Siegmund(1985)]{RobbinsSiegmund1971}
Herbert Robbins and David Siegmund.
\newblock A convergence theorem for non negative almost supermartingales and
  some applications.
\newblock In David Siegmund and Yi-Ching Yao, editors, \emph{Herbert Robbins
  Selected Papers}, volume~I, page ——. Springer, New York, NY, 1985.

\bibitem[Chernozhukov et~al.(2017)Chernozhukov, Chetverikov, Demirer, Duflo,
  Hansen, Newey, and Robins]{Chernozhukov2018DML}
Victor Chernozhukov, Denis Chetverikov, Mert Demirer, Esther Duflo, Christian
  Hansen, Whitney~K. Newey, and James Robins.
\newblock Double/debiased machine learning for treatment and structural
  parameters.
\newblock NBER Working Paper 23564, National Bureau of Economic Research, June
  2017.
\newblock URL \url{https://ssrn.com/abstract=2999543}.

\bibitem[Nguyen et~al.(2007)Nguyen, Wainwright, and
  Jordan]{NguyenWainwrightJordan2010}
XuanLong Nguyen, Martin~J Wainwright, and Michael Jordan.
\newblock Estimating divergence functionals and the likelihood ratio by
  penalized convex risk minimization.
\newblock In J.~Platt, D.~Koller, Y.~Singer, and S.~Roweis, editors,
  \emph{Advances in Neural Information Processing Systems}, volume~20. Curran
  Associates, Inc., 2007.
\newblock URL
  \url{https://proceedings.neurips.cc/paper_files/paper/2007/file/72da7fd6d1302c0a159f6436d01e9eb0-Paper.pdf}.

\bibitem[Kraskov et~al.(2004)Kraskov, St{\"o}gbauer, and
  Grassberger]{Kraskov2004}
Alexander Kraskov, Harald St{\"o}gbauer, and Peter Grassberger.
\newblock Estimating mutual information.
\newblock \emph{Physical Review E}, 69\penalty0 (6):\penalty0 066138, June
  2004.
\newblock \doi{10.1103/PhysRevE.69.066138}.

\bibitem[Efron(1982)]{Efron1982Bootstrap}
Bradley Efron.
\newblock \emph{The Jackknife, the Bootstrap, and Other Resampling Plans}.
\newblock Number~38 in CBMS-NSF Regional Conference Series in Applied
  Mathematics. Society for Industrial and Applied Mathematics (SIAM),
  Philadelphia, PA, USA, 1982.
\newblock ISBN 0898711797.

\bibitem[Minsker(2018)]{Minsker2018}
Stanislav Minsker.
\newblock Sub-gaussian estimators of the mean of a random matrix with
  heavy-tailed entries.
\newblock \emph{Annals of Statistics}, 46\penalty0 (6A):\penalty0 2871--2903,
  2018.
\newblock \doi{10.1214/17-AOS1642}.

\bibitem[Catoni(2012)]{Catoni2012}
Olivier Catoni.
\newblock Challenging the empirical mean and empirical variance: A deviation
  study.
\newblock \emph{Annales de l'I.H.P. Probabilités et statistiques}, 48\penalty0
  (4):\penalty0 1148--1185, 2012.
\newblock \doi{10.1214/11-AIHP454}.

\bibitem[Villani(2009)]{Villani2009}
C{\'e}dric Villani.
\newblock \emph{Optimal Transport: Old and New}, volume 338 of
  \emph{Grundlehren der mathematischen Wissenschaften}.
\newblock Springer Berlin Heidelberg, 2009.
\newblock ISBN 978-3-540-71049-3.

\bibitem[Pedregosa et~al.(2011)Pedregosa, Varoquaux, Gramfort, Michel, Thirion,
  Grisel, Blondel, Prettenhofer, Weiss, Dubourg, Vanderplas, Passos,
  Cournapeau, Brucher, Perrot, and Duchesnay]{Pedregosa2011}
Fabian Pedregosa, Ga{\"e}l Varoquaux, Alexandre Gramfort, Vincent Michel,
  Bertrand Thirion, Olivier Grisel, Mathieu Blondel, Peter Prettenhofer, Ron
  Weiss, Vincent Dubourg, Jake Vanderplas, Alexandre Passos, David Cournapeau,
  Matthieu Brucher, M.~Perrot, and {\'E}douard Duchesnay.
\newblock Scikit-learn: Machine learning in python.
\newblock \emph{Journal of Machine Learning Research}, 12:\penalty0 2825--2830,
  2011.

\end{thebibliography}

\appendix

\section{Proofs of Section~\ref{sec:properties}}
\label{app:proofs-properties}

\begin{proof}[Proof of Proposition~\ref{prop:bounded}]
	\textit{Ratio form.}
	By Definition~\ref{def:IE}, $\Score(\varphi;N)\in[0,\Score_{\mathrm{ref}}(N)]$ with $\Score_{\mathrm{ref}}(N)>0$, so
	\[
	0 \le \frac{\Score(\varphi;N)}{\Score_{\mathrm{ref}}(N)} \le 1.
	\]
	
	\textit{Calibrated-difference form.}
	If $\Score(\varphi;N)\in[\Score_{\min}(N),\Score_{\mathrm{ref}}(N)]$ with $\Score_{\min}(N) < \Score_{\mathrm{ref}}(N)$, then
	\[
	E(\varphi;N)
	= 1 - \frac{\Score_{\mathrm{ref}}(N)-\Score(\varphi;N)}
	{\Score_{\mathrm{ref}}(N)-\Score_{\min}(N)}
	= a\,\Score(\varphi;N)+b,
	\]
	where $a=1/(\Score_{\mathrm{ref}}(N)-\Score_{\min}(N))>0$ and
	$b=-\Score_{\min}(N)/(\Score_{\mathrm{ref}}(N)-\Score_{\min}(N))$.
	This affine map sends $[\Score_{\min}(N),\Score_{\mathrm{ref}}(N)]$ bijectively and order-preservingly to $[0,1]$. \qedhere
\end{proof}

\begin{proof}[Proof of Proposition~\ref{prop:continuity}]
	Fix $N$.
	
	\textit{Ratio form.}
	If $\Score(\cdot;N)$ is $\tau$-continuous on $\Phi$ and $\Score_{\mathrm{ref}}(N)>0$ is constant in $\varphi$, then
	\[
	E(\cdot;N)=\Score(\cdot;N)/\Score_{\mathrm{ref}}(N)
	\]
	is obtained by multiplication by a positive constant, and is $\tau$-continuous. The same argument applies for lower semicontinuity.
	
	\textit{Calibrated-difference form.}
	In this normalization $E(\cdot;N)$ is a positive affine transform of $\Score(\cdot;N)$. Positive affine maps preserve continuity and lower semicontinuity on topological vector spaces~\cite[Prop.~1.1]{RockafellarWets1998}. \qedhere
\end{proof}

\begin{proof}[Proof of Proposition~\ref{prop:dpi}]
	Let $Z=\varphi(X)$ and $Z'=T(Z)$ for an admissible post-map $T$. By assumption,
	\(
	\Score(T\!\circ\!\varphi;N)\le \Score(\varphi;N)
	\)
	for all $N$, which holds for $f$-divergence and mutual-information scores under Markov post-processing~\cite[Ch.~3]{CsiszarKorner2011}. Hence
	\[
	E(T\!\circ\!\varphi;N)
	= \frac{\Score(T\!\circ\!\varphi;N)}{\Score_{\mathrm{ref}}(N)}
	\le \frac{\Score(\varphi;N)}{\Score_{\mathrm{ref}}(N)}
	= E(\varphi;N),
	\]
	and the same inequality holds in the calibrated-difference form by applying the common positive affine map. \qedhere
\end{proof}

\begin{proof}[Proof of Proposition~\ref{prop:invariance}]
	Let $\mathcal{G}$ be the admissible invariance group. If $\Score(g\!\circ\!\varphi;N)=\Score(\varphi;N)$ for all $g\in\mathcal{G}$, then
	\[
	E(g\!\circ\!\varphi;N)
	= \frac{\Score(g\!\circ\!\varphi;N)}{\Score_{\mathrm{ref}}(N)}
	= \frac{\Score(\varphi;N)}{\Score_{\mathrm{ref}}(N)}
	= E(\varphi;N).
	\]
	In the calibrated-difference normalization, $E$ is a positive affine transform of $\Score$, so it inherits the same invariances. This is the standard invariance transfer principle in equivariant decision problems~\cite[Sec.~1.5]{LehmannCasella1998}. \qedhere
\end{proof}

\section{Proofs of Section~\ref{sec:relations}}
\label{app:proofs-relations}

\begin{proof}[Proof of Theorem~\ref{thm:mi}]
	Combining~\eqref{eq:calib} with the upper bound in~\eqref{eq:ref-calib} gives
	\[
	E(\varphi;N)
	= \frac{\Score(\varphi;N)}{\Score_{\mathrm{ref}}(N)}
	\ge \frac{\alpha_N\,\MI(Z;Y)}{d_N\,\MI(X;Y)}
	= \frac{\alpha_N}{d_N}\frac{\MI(Z;Y)}{\MI(X;Y)}.
	\]
	Combining~\eqref{eq:calib} with the lower bound in~\eqref{eq:ref-calib} gives
	\[
	E(\varphi;N)
	\le \frac{\beta_N\,\MI(Z;Y)+\gamma_N}{c_N\,\MI(X;Y)}
	= \frac{\beta_N}{c_N}\frac{\MI(Z;Y)}{\MI(X;Y)}
	+ \frac{\gamma_N}{c_N\MI(X;Y)}.
	\]
	Set $a_N=\alpha_N/d_N$, $b_N=\beta_N/c_N$, and $\varepsilon_N=\gamma_N/(c_N\MI(X;Y))$ to obtain the compact form. The DPI for $Y\to X\to Z$ ensures $0\le \MI(Z;Y)\le \MI(X;Y)$~\cite{SasonVerdu2016,PolyanskiyWu2019}. \qedhere
\end{proof}

\begin{proof}[Proof of Theorem~\ref{thm:fisher}]
	Under differentiability in quadratic mean, local asymptotic normality at $\theta^\star$ gives
	\[
	\log\frac{p_{\theta^\star+h}}{p_{\theta^\star}}
	= h^\top s_{\theta^\star}-\tfrac{1}{2}h^\top \Fisher(\theta^\star)h
	+ o(\|h\|^2),
	\]
	with score $s_{\theta^\star}$ and Fisher information $\Fisher(\theta^\star)$~\cite{LeCam1986}. For $Z=\varphi(X)$, the $L^2(P_{\theta^\star})$ projection of $s_{\theta^\star}$ onto the $\sigma(Z)$-measurable subspace is $\E[s_{\theta^\star}\mid Z]$~\cite{AmariNagaoka2000}. Denoting the corresponding orthogonal projection by $\Pi_\varphi$, the induced curvature is
	\[
	h^\top \Pi_\varphi \Fisher(\theta^\star)\Pi_\varphi^\top h,
	\]
	so any locally Fisher-driven score admits the expansion
	\[
	\Score(\varphi;N)
	= h^\top \Pi_\varphi \Fisher(\theta^\star)\Pi_\varphi^\top h
	+ o(\|h\|^2).
	\]
	If $\Score_{\mathrm{ref}}(N)\asymp h^\top \Fisher(\theta^\star)h$, then
	\[
	E(\varphi;N)
	= \frac{h^\top \Pi_\varphi \Fisher(\theta^\star)\Pi_\varphi^\top h}
	{h^\top \Fisher(\theta^\star)h}
	+ o(1).
	\]
	Averaging over directions, or equivalently using $\mathrm{tr}(A)=\E[u^\top A u]$ for $u$ uniform on the unit sphere, yields
	\[
	E(\varphi;N)
	= \frac{\mathrm{tr}(\Pi_\varphi \Fisher(\theta^\star))}
	{\mathrm{tr}(\Fisher(\theta^\star))}+o(1).
	\]
	See also~\cite{Kay1993,VanTrees2001}. \qedhere
\end{proof}

\begin{proof}[Proof of Proposition~\ref{prop:vgib}]
	Let $\Score(\varphi;N)=U_\beta(\varphi)=\MI(Z;Y)-\beta\,\MI(Z;X)$, or any positive affine transform thereof. Each mutual-information term satisfies DPI and admissible invariances. If $C=\sup_\psi U_\beta(\psi)\in(0,\infty)$, then
	\[
	E(\varphi;N)=\frac{U_\beta(\varphi)}{C}
	\]
	is a positive rescaling of $U_\beta$ and inherits the axioms in Section~\ref{sec:properties}. When only upper and lower brackets for $C$ are available, the calibrated-difference form yields a positive affine transform of $U_\beta$ and the same conclusion holds. \qedhere
\end{proof}

\section{Proofs of Section~\ref{sec:asymptotics}}
\label{app:proofs-asymptotics}

\begin{proof}[Proof of Theorem~\ref{thm:consistency}]
	By the Glivenko–Cantelli assumption,
	\[
	\sup_{\varphi\in\Phi}|\Score(\varphi;N)-\Score_\infty(\varphi)| \to 0
	\quad\text{almost surely.}
	\]
	Fix $\varphi$. Then $\Score(\varphi;N)\to\Score_\infty(\varphi)$ almost surely. By assumption $\Score_{\mathrm{ref}}(N)\to\Score_{\mathrm{ref},\infty}>0$ almost surely. The ratio and calibrated-difference maps are continuous on their domains, so $E(\varphi;N)\to E_\infty(\varphi)$ almost surely by the continuous mapping theorem. \qedhere
\end{proof}

\begin{proof}[Proof of Proposition~\ref{prop:rates}]
	Standard symmetrization and Rademacher-contraction arguments~\cite{BartlettMendelson2002,Wainwright2019} yield, with probability at least $1-\delta$,
	\[
	\sup_{\varphi\in\Phi}|\Score(\varphi;N)-\Score_\infty(\varphi)|
	\lesssim \mathrm{Rad}_N(\{s_\varphi\})
	+ \sqrt{\tfrac{\log(1/\delta)}{N}},
	\]
	under sub-Gaussian or sub-exponential conditions, with an additional linear term in $\log(1/\delta)/N$ in the sub-exponential case~\cite{BoucheronLugosiMassart2013}. Dividing by $\Score_{\mathrm{ref},\infty}>0$ gives the stated bound for $E$.
	
	Under a Bernstein condition, localized empirical-process techniques (peeling and fixed-point arguments) give fast rates governed by the localized complexity of $\Phi$~\cite{BoucheronLugosiMassart2013,Tsybakov2004,Wainwright2019}. The result transfers to $E$ after normalization. \qedhere
\end{proof}

\begin{proof}[Proof of Proposition~\ref{prop:dynamics}]
	Let $M_t=\Score(\varphi_t;N)/\Score_{\mathrm{ref}}(N)$. By the assumed nonnegative expected improvement,
	\[
	\E[M_{t+1}-M_t\mid\mathcal{F}_t]
	= \frac{\E[\Score(\varphi_{t+1};N)-\Score(\varphi_t;N)\mid\mathcal{F}_t]}
	{\Score_{\mathrm{ref}}(N)} \ge 0,
	\]
	so $\{M_t\}$ is a submartingale. The bounded increment condition implies
	\(
	|M_{t+1}-M_t|\le c/\Score_{\mathrm{ref}}(N)
	\),
	and Azuma–Hoeffding gives the stated tail bound.
	
	If
	\(
	\sum_t \E[(M_{t+1}-M_t)^-\mid\mathcal{F}_t]<\infty
	\)
	almost surely, the convergence of $M_t$ follows from the Robbins–Siegmund theorem~\cite{RobbinsSiegmund1971}. The calibrated-difference normalization is obtained by a positive affine transform of $M_t$. \qedhere
\end{proof}

\section{Proofs of Section~\ref{sec:estimation}}
\label{app:proofs-estimation}

\paragraph{Notation.}
For $\varphi\in\Phi$, write
\[
\widehat{\Score}(\varphi;N)
= \frac{1}{N}\sum_{i=1}^N s_\varphi(X_i,Y_i),
\qquad
\Score(\varphi)=\E[s_\varphi(X,Y)].
\]
The reference estimator $\widehat{\Score}_{\mathrm{ref}}(N)$ satisfies $\widehat{\Score}_{\mathrm{ref}}(N)\to\Score_{\mathrm{ref},\infty}\in(0,\infty)$.  
Define $\widehat{E}(\varphi;N)=\widehat{\Score}(\varphi;N)/\widehat{\Score}_{\mathrm{ref}}(N)$ for the ratio form; the calibrated-difference case follows by a positive affine transform. Constants may change from line to line.

\subsection{Auxiliary lemmas}

\begin{lemma}[Sub-exponential empirical-process deviation]
	\label{lem:subexp-ep}
	Suppose that for all $\varphi\in\Phi$, $s_\varphi(X,Y)$ is centered and sub-exponential with parameters $(\nu,b)$. Then for any $\delta\in(0,1)$,
	\[
	\sup_{\varphi\in\Phi}
	\big|\widehat{\Score}(\varphi;N)-\Score(\varphi)\big|
	\le
	C_1\,\mathfrak{R}_N(\Phi)
	+
	C_2\,\sqrt{\frac{\log(2/\delta)}{N}}
	+
	C_3\,\frac{\log(2/\delta)}{N}
	\]
	with probability at least $1-\delta$, where $\mathfrak{R}_N(\Phi)$ denotes a localized Rademacher complexity or entropy integral for $\{s_\varphi\}$ and $C_i=C_i(\nu,b)$.
	\emph{Proof.}
	Apply symmetrization and contraction to the centered class and conclude with sub-exponential Bernstein bounds~\cite{BoucheronLugosiMassart2013}. \qed
\end{lemma}

\begin{lemma}[Variational MI critic class]
	\label{lem:var-mi}
	Let $\Score(\varphi)$ be the optimum of a variational objective over critics $T\in\mathcal{T}$ (e.g.\ NWJ or DV), and let $\widehat{\Score}(\varphi)$ be its cross-fitted empirical version. If $\mathcal{T}$ has complexity $\mathfrak{R}_N(\mathcal{T})$, then for any $\delta\in(0,1)$,
	\[
	\sup_{\varphi\in\Phi}
	\big|\widehat{\Score}(\varphi)-\Score(\varphi)\big|
	\le
	C\big(\mathfrak{R}_N(\Phi)+\mathfrak{R}_N(\mathcal{T})\big)
	+
	C\sqrt{\tfrac{\log(2/\delta)}{N}}
	+
	C\tfrac{\log(2/\delta)}{N}
	\]
	with probability at least $1-\delta$.
	\emph{Proof.}
	Control the supremum over the product class $\Phi\times\mathcal{T}$ and use cross-fitting to avoid dependence between critic training and evaluation~\cite{NguyenWainwrightJordan2010}. \qed
\end{lemma}

\begin{lemma}[Ratio concentration via delta method]
	\label{lem:ratio}
	Let $(U_N,V_N)\to(\mu,\nu)$ in probability with $\nu>0$. Suppose $U_N$ and $V_N$ admit deviations $a_N(\delta),b_N(\delta)\to 0$ and assume independence. Then, for $N$ large enough and any $\delta\in(0,1)$,
	\[
	\left|
	\frac{U_N}{V_N}-\frac{\mu}{\nu}
	\right|
	\le
	\frac{a_N(\delta)}{\nu-b_N(\delta)}
	+
	\frac{|\mu|\,b_N(\delta)}{\nu(\nu-b_N(\delta))}
	\]
	with probability at least $1-2\delta$.
	\emph{Proof.}
	Write
	\[
	\frac{U_N}{V_N}-\frac{\mu}{\nu}
	=
	\frac{U_N-\mu}{V_N}
	+
	\mu\Big(\frac{1}{V_N}-\frac{1}{\nu}\Big),
	\]
	and bound each term on the event $\{|V_N-\nu|\le b_N(\delta)\}$. \qed
\end{lemma}

\begin{lemma}[Robust Lipschitz perturbation bound]
	\label{lem:lipschitz}
	If $s_\varphi$ is $L_s$–Lipschitz in $(x,y)$ uniformly over $\Phi$, then for per-sample perturbations with $\|(x_i',y_i')-(x_i,y_i)\|\le\epsilon$,
	\[
	\big|\widehat{\Score}_\epsilon(\varphi;N)-\widehat{\Score}(\varphi;N)\big|
	\le
	L_s\,\epsilon,
	\qquad
	\forall\,\varphi\in\Phi.
	\]
	If the data-generating distribution shifts within 1-Wasserstein distance $\rho$, then
	\[
	|\Score_\rho(\varphi)-\Score(\varphi)|\le L_s\,\rho.
	\]
	\emph{Proof.}
	For empirical perturbations, apply the Lipschitz bound to each term and average. For population shifts, use the Kantorovich–Rubinstein duality for $W_1$~\cite{Villani2009}. \qed
\end{lemma}

\subsection{Proof of Theorem~\ref{thm:concentration}}

\begin{proof}[Proof of Theorem~\ref{thm:concentration}]
	Consider the ratio form with $U_N=\widehat{\Score}(\varphi;N)$ and $V_N=\widehat{\Score}_{\mathrm{ref}}(N)$, whose limits are $\mu=\Score(\varphi)$ and $\nu=\Score_{\mathrm{ref},\infty}>0$.
	
	Lemma~\ref{lem:subexp-ep} gives, with probability at least $1-\delta$,
	\[
	|U_N-\mu|
	\le
	C_1\,\mathfrak{R}_N(\Phi)
	+
	C_2\sqrt{\tfrac{\log(2/\delta)}{N}}
	+
	C_3\tfrac{\log(2/\delta)}{N}.
	\]
	By cross-fitting and a standard law of large numbers (or Bernstein bound) for the reference,
	\[
	|V_N-\nu|
	\le
	C_4\sqrt{\tfrac{\log(2/\delta)}{N}}
	+
	C_5\tfrac{\log(2/\delta)}{N}
	\]
	with probability at least $1-\delta$. Applying Lemma~\ref{lem:ratio} and absorbing constants yields the stated deviation bound for $\widehat{E}(\varphi;N)-E(\varphi;N)$.
	
	When $\Score$ is a variational MI estimator, Lemma~\ref{lem:var-mi} replaces Lemma~\ref{lem:subexp-ep}, introducing the additional $\mathfrak{R}_N(\mathcal{T})$ term. The calibrated-difference normalization follows by a positive affine transform and Slutsky’s theorem. \qedhere
\end{proof}

\subsection{Proof of Proposition~\ref{prop:robust}}

\begin{proof}[Proof of Proposition~\ref{prop:robust}]
	\textit{Empirical perturbations.}
	If $\|(x_i',y_i')-(x_i,y_i)\|\le\epsilon$, Lemma~\ref{lem:lipschitz} gives
	\[
	|\widehat{\Score}_\epsilon(\varphi;N)-\widehat{\Score}(\varphi;N)|
	\le
	L_s\,\epsilon.
	\]
	Hence
	\[
	|\widehat{E}_\epsilon(\varphi;N)-\widehat{E}(\varphi;N)|
	=
	\left|
	\frac{\widehat{\Score}_\epsilon-\widehat{\Score}}
	{\widehat{\Score}_{\mathrm{ref}}(N)}
	\right|
	\le
	\frac{L_s\epsilon}{\Score_{\mathrm{ref}}(N)},
	\]
	using the positive reference term.
	
	\textit{Distribution shift.}
	If $W_1(P,P')=\rho$, Lemma~\ref{lem:lipschitz} yields
	\[
	|\Score_\rho(\varphi)-\Score(\varphi)|\le L_s\rho,
	\]
	so
	\[
	|E_\rho(\varphi;N)-E(\varphi;N)|
	\le
	\frac{L_s\rho}{\Score_{\mathrm{ref},\infty}}.
	\]
	
	\textit{Concentration under perturbations.}
	Since perturbations contribute at most $L_s\epsilon$ or $L_s\rho$ to the score, the constants in Theorem~\ref{thm:concentration} change only by multiplicative factors; the qualitative rate remains the same~\cite{Villani2009,BoucheronLugosiMassart2013}. \qedhere
\end{proof}

\subsection{Algorithmic Computation of Interpretive Efficiency}
\label{sec:algorithm}

\begin{algorithm}[!h]
	\caption{Computation of Interpretive Efficiency $E(\varphi;N)$}
	\label{alg:IE}
	\begin{algorithmic}[1]
		\Require  
		Dataset $\Data=\{(x_i,y_i)\}_{i=1}^N$;  
		interpretive map $\varphi:\mathcal{X}\to\mathcal{Z}$;  
		reference scoring rule $\Score_{\mathrm{ref}}$.
		\Ensure  
		Empirical efficiency estimate $\widehat{E}(\varphi;N)$.
		
		\State \textbf{Interpretive score.}
		Compute
		\[
		\widehat{\Score}(\varphi;N)
		=
		\frac{1}{N}\sum_{i=1}^N s_\varphi(x_i,y_i),
		\]
		where $s_\varphi$ encodes the chosen task-specific contribution (e.g., MI, Fisher curvature, or risk reduction).
		
		\State \textbf{Reference score.}
		Compute $\widehat{\Score}_{\mathrm{ref}}(N)$ as the oracle or calibrated upper-bound score (e.g., identity-channel MI or full-information decoder).
		
		\State \textbf{Normalization.}
		Set
		\[
		\widehat{E}(\varphi;N)
		=
		\frac{\widehat{\Score}(\varphi;N)}
		{\widehat{\Score}_{\mathrm{ref}}(N)}.
		\]
		If $\Score_{\mathrm{ref}}$ is only bracketed, use the calibrated-difference form.
		
		\State \textbf{Cross-fitting (optional).}
		Partition $\Data$ into $K$ folds. For each fold $k$, train any first-stage components (e.g., $\varphi$ or critic $T$) on $\Data\setminus\Data_k$ and evaluate $s_\varphi$ on $\Data_k$, then aggregate
		\[
		\widehat{\Score}_{\mathrm{CF}}(\varphi;N)
		=
		\frac{1}{K}
		\sum_{k=1}^K
		\frac{1}{|\Data_k|}
		\sum_{(x,y)\in\Data_k}
		s_\varphi(x,y).
		\]
		
		\State \textbf{Bias correction (optional).}
		Apply jackknife, median-of-means, or Catoni-type truncation to reduce finite-sample bias and heavy-tail effects if needed.
		
		\State \textbf{Output.}
		Return $\widehat{E}(\varphi;N)$ together with a confidence radius
		\[
		r_N(\delta)
		=
		C
		\sqrt{
			\frac{\mathrm{comp}(\varphi)+\log(1/\delta)}
			{N}
		},
		\]
		consistent with Theorem~\ref{thm:concentration}, where $\mathrm{comp}(\varphi)$ denotes the relevant complexity measure (e.g., localized Rademacher complexity).
	\end{algorithmic}
\end{algorithm}

\section{Additional Experimental Details}
\label{app:exp-details}

This appendix collects the implementation details needed to reproduce the validation experiments in Section~\ref{sec:validation}. Unless stated otherwise, the same protocol is used across all channels and ablations.

\subsection{Datasets, sample sizes, and splits}

\paragraph{Digits.}
We use the \texttt{sklearn} Digits dataset~\citep{Pedregosa2011} with $N_{\mathrm{train}}=1{,}294$ and $N_{\mathrm{test}}=503$ after the standard train--test split provided by the library. Images are $8\times 8$ grayscale, flattened to $\mathbb{R}^{64}$ and standardized featurewise (mean zero, unit variance) on the training set only.

\paragraph{Sinusoids.}
For the synthetic signals, we generate $N_{\mathrm{train}}=4{,}000$ and $N_{\mathrm{test}}=1{,}000$ waveforms. Each sample is a length-$T$ time series with sampling rate $f_s=128$\,Hz and duration $T=1$\,s. Class $0$ signals use base frequency $5$\,Hz, class $1$ signals use $9$\,Hz. For each sample we draw a random phase $\phi\sim\mathrm{Unif}[0,2\pi)$, amplitude $A\sim\mathrm{Unif}[0.8,1.2]$, apply mild amplitude modulation with a low-frequency envelope ($0.5$--$1$\,Hz), and add Gaussian noise with signal-to-noise ratio between $15$ and $20$\,dB. Train--test splits are fixed once and reused across all runs.

All results reported in Section~\ref{sec:validation} are averaged over three-fold cross-validation on the training set; test performance is computed once on the held-out test set where applicable.

\subsection{Interpretive channels and ablations}

\paragraph{Digits.}
We consider three interpretive channels:
\begin{enumerate}[label=(D\arabic*)]
	\item \emph{Identity:} $Z=X\in\mathbb{R}^{64}$.
	\item \emph{PCA-$k$:} principal components trained on the training set with $k\in\{4,8,16,32,64\}$, retaining the top eigenvectors and projecting $X$ to $\mathbb{R}^k$.
	\item \emph{Random projection-$k$:} Gaussian random projection with $k\in\{4,8,16,32,64\}$ using an orthogonally normalized matrix with entries $\mathcal{N}(0,1/k)$.
\end{enumerate}
All embeddings are standardized before mutual-information estimation and classifier training. The extended PCA vs.\ random-projection sweep in Table~\ref{tab:digits-ie} uses the same construction.

\paragraph{Sinusoids.}
We consider:
\begin{enumerate}[label=(S\arabic*)]
	\item \emph{FFT-top-20:} magnitude of the top $20$ frequency bins (excluding DC), sorted by energy and fixed across runs.
	\item \emph{Downsample-32:} uniform subsampling of the time series to $32$ time points.
	\item \emph{Random projection-16:} Gaussian random projection of the full waveform to $\mathbb{R}^{16}$, constructed as in the Digits setting.
\end{enumerate}
These choices match the main-text description in Section~\ref{sec:val-design} and are designed to induce controlled information retention or degradation.

\subsection{Mutual-information estimation}

For all channels we define
\[
\Score(\varphi;N)=\widehat{I}(Z;Y), \qquad Z=\varphi(X),
\]
and set the reference to $\Score_{\mathrm{ref}}(N)=\widehat{I}(X;Y)$ with $Z=X$.

\paragraph{Critic architecture.}
The primary mutual-information estimator is a neural NWJ/DV-style lower bound. The critic $T_\omega(z,y)$ is a two-layer multilayer perceptron with ReLU activations and hidden width $256$:
\[
(z,y) \mapsto \mathrm{MLP}_{256,256}(z,y) \to \mathbb{R}.
\]
Inputs are concatenated one-hot labels $y$ and standardized features $z$. We train $T_\omega$ using Adam with learning rate $10^{-3}$, batch size $256$, and $5{,}000$ gradient steps per run. Early stopping on a held-out validation subset prevents overfitting in low-sample regimes.

\paragraph{Estimator variants.}
For robustness, we also compute a Donsker--Varadhan estimator and a $k$NN estimator (with $k\in\{5,10\}$) on a subset of runs. These are used only for sanity checks and ablations; unless explicitly labeled otherwise, the main figures and tables report the neural NWJ estimate. The ensemble behaviour of these estimators under the calibration assumptions in Section~\ref{sec:relations} is consistent with the bounds in Theorem~\ref{thm:mi}.

\subsection{Classifiers, training, and robustness probes}

\paragraph{Base classifier.}
The main classifier is multinomial logistic regression with $\ell_2$ regularization tuned by cross-validation on the training set. Features are standardized after each interpretive mapping.

\paragraph{Alternative classifier.}
For a subset of settings we replace logistic regression by an RBF SVM with kernel width and regularization selected by grid search. As reported in Section~\ref{sec:val-robust}, this increases absolute accuracy but does not change the ordering of $E(\varphi;N)$ across channels.

\paragraph{Robustness experiments.}
For the Digits robustness experiment (Table~\ref{tab:digits-ie}), we add i.i.d.\ Gaussian noise $\mathcal{N}(0,\sigma^2)$ to input pixels at test time with a fixed SNR range chosen to reduce clean accuracy by roughly $10$--$20$ percentage points for the identity channel. The robustness gap is defined as the difference between clean and noisy test accuracy.

\subsection{Random seeds and repetitions}

All experiments use a fixed base random seed $s_0=42$. For results that involve stochastic elements (random projections, critic initialization, optimization, and noise in robustness tests), we run $R=10$ independent repetitions with seeds $s_r=s_0+r$ and report the mean and standard error across repetitions. Train--test splits are held fixed; cross-validation folds are re-shuffled per repetition.

\subsection{Uncertainty quantification}

\paragraph{Accuracy.}
For each configuration we report the mean cross-validated accuracy and its standard error across the $R$ repetitions. When confidence intervals are shown, we use normal-approximation $95\%$ intervals
\[
\hat{p} \pm 1.96 \,\widehat{\mathrm{SE}}(\hat{p}),
\]
where $\hat{p}$ is the mean accuracy and $\widehat{\mathrm{SE}}$ is the empirical standard deviation of $\hat{p}$ across repetitions divided by $\sqrt{R}$.

\paragraph{Interpretive efficiency.}
For $E(\varphi;N)$ we similarly report the mean and standard error across seeds. The dispersion decreases with $N$ at a rate consistent with Theorem~\ref{thm:concentration}. When plotting bars with error bars, the vertical bars show the mean, and the whiskers show $\pm$ one empirical standard error; underlying values are logged in the exported CSV files.

\paragraph{Reproducibility.}
All hyperparameters, including learning rates, critic widths, number of optimization steps, and noise levels, are exposed as command-line flags in the public implementation. The code corresponding to Algorithm~\ref{alg:IE}, the validation experiments, and all ablations runs as a single script that produces the CSV tables and PDF figures used in Section~\ref{sec:validation}.

\end{document}